\documentclass[letterpaper]{article}
\usepackage{aaai2026}
\copyrighttext{Accepted for publication in the Proceedings of the 36th International Conference on Automated Planning and Scheduling (ICAPS 2026).}
\mathchardef\mhyphen="2D
\usepackage{amsmath}
\usepackage{amsfonts}
\usepackage{subcaption}
\usepackage{cuted}
\usepackage{algpseudocode}
\usepackage[english]{babel}
\usepackage{amsthm}
\usepackage{thmtools}
\usepackage{latexsym}

\theoremstyle{definition}

\theoremstyle{plain}
\newtheorem{theorem}{Theorem}
\newtheorem{corollary}{Corollary}
\newtheorem{proposition}{Proposition}
\newtheorem{lemma}{Lemma}

\theoremstyle{plain}

\usepackage{times}  
\usepackage{helvet}  
\usepackage{courier}  
\usepackage[hyphens]{url}  
\usepackage{graphicx} 
\usepackage{natbib}  
\usepackage{caption} 
\usepackage{algorithm}

\usepackage{algorithmicx}

\usepackage{newfloat}
\usepackage{listings}
\DeclareCaptionStyle{ruled}{labelfont=normalfont,labelsep=colon,strut=off} 
\floatstyle{ruled}
\newfloat{listing}{tb}{lst}{}
\floatname{listing}{Listing}

\title{Leveraging the Value of Information in POMDP Planning}
\author {
    Zakariya Laouar\textsuperscript{\rm 1},
    Qi Heng Ho\textsuperscript{\rm 2},
    Zachary Sunberg\textsuperscript{\rm 1}
}
\affiliations {
    \textsuperscript{\rm 1}Department of Aerospace Engineering Sciences, University of Colorado Boulder\\
    \textsuperscript{\rm 2}Kevin T. Crofton Department of Aerospace and Ocean Engineering, Virginia Tech\\
    zakariya.laouar@colorado.edu, qihengho@vt.edu, zachary.sunberg@colorado.edu
}

\begin{document}

\maketitle

\begin{abstract}
Partially observable Markov decision processes (POMDPs) offer a principled formalism for planning under state and transition uncertainty. Despite advances made towards solving large POMDPs, obtaining performant policies under limited planning time remains a major challenge due to the curse of dimensionality and the curse of history. For many POMDP problems, the value of information (VOI) --- the expected performance gain from reasoning about observations --- varies over the belief space. We introduce a dynamic programming framework that exploits this structure by conditionally processing observations based on the value of information at each belief. Building on this framework, we propose Value of Information Monte Carlo planning (VOIMCP), a Monte Carlo Tree Search algorithm that allocates computational effort more efficiently by selectively disregarding observation information when the VOI is low, avoiding unnecessary branching of observations. We provide theoretical guarantees on the near-optimality of our VOI reasoning framework and derive non-asymptotic convergence bounds for VOIMCP. Simulation evaluations demonstrate that VOIMCP outperforms baselines on several POMDP benchmarks.
\end{abstract}

\section{Introduction}
Decision making in uncertain environments presents a critical challenge. The partially observable Markov decision process (POMDP) framework provides a principled way to plan in such environments \citep{astrom1965optimal, smallwood1973optimal}. Their applications include robotics \cite{Lauri2023POMDProbotics}, resource management \cite{PAPAKONSTANTINOU2014maintenance}, human-computer interaction \cite{Chen2020trustpomdp}, and medical diagnosis \cite{ayer_or_2012}. However, finding even approximate solutions for POMDPs is computationally intractable \cite{madani2003undecidability} for two main reasons. 

Firstly, the computational complexity of the problem scales exponentially with the dimensionality of the state, action, and observation spaces. Secondly, since the state is not directly observable, optimal decisions may depend on the entire history of previous actions and observations, the space of which grows exponentially with the planning horizon. These two fundamental challenges are known as the \emph{curse of dimensionality} and \emph{curse of history}. In this work, we aim to alleviate the curse of history by developing a framework that enables planners to dynamically prune areas of the action-observation histories while maintaining performance guarantees.

To address the challenges of POMDP planning, many practical solvers have been proposed in the past few decades \citep{shani2013survey, silver2010pomcp, Kurniawati-RSS08-SARSOP, Smith2005HSVI2, despot, sunberg2018pomcpow}. Sampling-based approximate solvers such as \citet{silver2010pomcp, sunberg2018pomcpow, despot} have made good progress, especially regarding the curse of dimensionality in the state space, by employing particles in search and exploring a subset of histories in the search tree. Nonetheless, problems with large observation spaces still remain relatively difficult for state-of-the-art planners. In such problems, many planners tend to generate shallow trees, which lead to myopic plans that fail to capture long-horizon consequences of such plans.

For many POMDP problems, the value of information (VOI) \cite{wei2024value, flaspohler2020belief}, which is the performance gain from fully reasoning about all possible observations from a belief or history, varies significantly over the belief space. This suggests that it may not always be necessary to consider all possible observations at every decision point: when the VOI is small, ignoring observations may incur negligible performance loss while substantially reducing the effective branching on observations.

In this work, we introduce an adaptive VOI framework that formalizes this structure by selectively processing observations based on the VOI at a belief. When VOI is low, this framework disregards uninformative observations. We further derive an upper bound on the suboptimality incurred by this adaptive VOI framework relative to the original POMDP value function. This framework offers a principled mechanism that formalizes the trade-off between open-loop execution (disregarding observations) and closed-loop execution (reasoning over observations). Then, we introduce the VOI-POMDP, a structural transformation that encodes this meta-level choice directly into the problem dynamics.

Leveraging this theoretical framework, we propose value of information Monte Carlo planning (VOIMCP), an MCTS-based POMDP planning algorithm that selectively branches on observations when useful, enabling more efficient planning for POMDPs with lower value of information by implicitly pruning the history space and reducing the effective branching factor on observations. We prove the non-asymptotic theoretical properties of VOIMCP, and show that we can also recover the optimal POMDP policy. Simulation experiments across a series of benchmarks show that incorporating our VOI framework in online time-constrained planning leads to better performance by enabling long-horizon reasoning.
\section{Related Work}
Incorporating the value of information in decision making can be traced back to work on information value theory \citep{howard1966information}. In this line of work, the value of computation is treated such that computations are selected according to the expected improvement in decision quality resulting from their execution. \citet{russel1988game} and \citet{russel1991metareasoning} formulated a meta-level decision problem designed to enable optimal allocation of computation, while other works employ heuristic approximations \cite{hay2014selecting}. However, these approaches are limited to fully observable sequential decision making.

The use of the VOI for partially observable problems has been relatively less explored. \citet{wei2024value} discusses the value of information based on the difference between purely open-loop and purely closed-loop policies. In contrast, we consider a recursive formulation of VOI that enables adaptive switching between the two modalities.

A recent work by \citet{flaspohler2020belief} shows how VOI can be used to construct macro-actions in POMDPs. The authors propose a VOI framework similar to the one we introduce in this paper. However, while \citet{flaspohler2020belief} define a Bellman backup based on an absolute difference threshold between open-loop and closed-loop values, our formulation computes a relative percentage difference threshold of the closed-loop value. This formulation is more general and intuitive since the VOI reasoning does not depend on the specific POMDP instantiation, its reward structure or the magnitude of its optimal value. This enables adaptive reasoning. Additionally, we introduce a novel alternative POMDP representation that exposes the machinery relevant for VOI reasoning in planning. Finally, we leverage our framework in a variant of POMCP for sampling-based planning instead of a method to construct macro-actions.

A core challenge when using tree search methods to solve POMDPs lies in the search strategy: when to simulate empirically valuable actions vs. lesser explored actions. This exploration-exploitation tradeoff has been the subject of many works attempting to address the curse of history. Several approaches have incorporated heuristic information gathering procedures. \citet{alves2023ibpomcp, pokharel2024increasing} use belief and observation entropy heuristics to bias the policy towards entropy-reducing actions as a proxy for actions with high VOI. However, neither belief nor observation entropy alone can sufficiently characterize the true VOI at a belief. Instead, we tackle the exploration-exploitation tradeoff by explicitly reasoning over VOI. 

\citet{kim_refpomdp, Kim2025porp} have explored sampling actions based on a reference policy instead of fully expanding them at each node, showing promise in addressing challenges with long-horizon planning. While \citet{kim_refpomdp} focus on reducing the effective action branching by embedding a reference policy in the objective, our work can be seen as reducing the effective observation branching factor when the value gained by branching is low.

Our work is also related to abstraction \citep{ho2019value} and selective perception \cite{mccallum1996reinforcement}. In particular, taking an open-loop execution mode in our framework is equivalent to abstracting all observations into a single observation. In online POMDP planning, \citet{Wu2021adaOps} propose a belief packing procedure that merges similar beliefs, effectively creating a branch that corresponds to an abstracted observation. On the other hand, our methodology adaptively ignores observations based on the VOI rather than belief similarity. This distinction is critical because observations may generate statistically dissimilar beliefs that nonetheless support the same optimal policy. Whereas similarity-based metrics would force branching in these cases, our value-based approach identifies that the information is decision-irrelevant, enabling more aggressive pruning of the search space.
\section{Partially Observable Markov Decision Processes}
\label{sec: preliminaries}
A finite-horizon \emph{Partially Observable Markov Decision Process} (POMDP) can be formally represented as a tuple $\mathcal{P}$ = $(\mathcal{S}, \mathcal{A}, \mathcal{O}, \mathcal{T}, \mathcal{R}, \mathcal{Z}, D, b_0, \gamma)$, where:
$\mathcal{S}, \mathcal{A},$ and $\mathcal{O}$ are finite sets of states, actions, and observations, respectively.
$\mathcal{T} : \mathcal{S} \times \mathcal{A} \times \mathcal{S} \rightarrow [0,1]$ is the transition probability function,
$\mathcal{R} : \mathcal{S} \times \mathcal{A} \rightarrow [-R_{max}, R_
{max}]$, for some $R_{max} > 0$, is the immediate reward function,
$\mathcal{Z} : \mathcal{S} \times \mathcal{A} \times \mathcal{O} \rightarrow [0,1]$ is the probabilistic observation function, $D$ is the horizon of the problem, $b_0 \in \Delta(\mathcal{S})$ is the initial belief or probability distribution over states,
and $\gamma \in [0,1]$ is the discount factor.

A defining characteristic of POMDPs is that the state $s\in\mathcal{S}$ cannot be observed directly. Instead, after taking an action $a\in \mathcal{A}$, an observation $o\in \mathcal{O}$ is received which provides incomplete information about the state. Unlike in MDPs, where optimal actions depend on states alone, in POMDPs, the optimal action must be computed from a history $h_t \in \mathcal{H}$ at time $t$, which defines a sequence of action-observation pairs: $h_t \equiv \{b_0, a_1, o_1, ... , a_t, o_t\}$. Action-terminated histories are defined as $ha_{t+1} \equiv \{b_0, a_1, o_1, ... , a_t, o_t, a_{t+1}\}$. It is sufficient to summarize the information contained in $h_t$ as the belief $b_t(s) \equiv \mathbb{P}(s_t=s|h_t) \in \Delta(\mathcal{S})$. With a slight abuse of notation, we sometimes drop the subscript $t$ and express histories as $h$ and action-terminated histories as $ha$.

A general finite-horizon policy is a sequence $\pi = (\pi_D,\dots,\pi_{1})$ where each $\pi_d:\Delta(\mathcal{S}) \to \mathcal{A}$ describes how an agent behaves at any belief. 

The objective of the agent is to find a policy that maximizes the discounted sum of future rewards for each belief starting from an initial belief $b_0$. Formally, the optimal policy is defined as:
\begin{align}
  \pi^* \in \operatorname*{arg\,max}_{\pi} \mathbb{E}\Bigl[ \sum_{t=0}^{D-1}\gamma^t r(b_t, \pi_{D-t}(b_t)) \; \Big| \; b_0\Bigr],
\end{align}
where $b_t$ is the updated belief at time $t$.

Given an action and observation, a belief update can be performed using Bayes' rule:
\begin{align}
b'(s')  = \tau(b,a,o)(s') &=
  \frac{\mathcal{Z}(s',o,a)\;\sum_{s \in \mathcal{S}} \mathcal{T}(s, a, s')\, b(s)}
       {\mathbb{P}(o \mid b, a)}, \label{eq:belief_update_full}
\end{align}
where the probability of observing $o$ is
\begin{align}
    \mathbb{P}(o \mid b, a) = \sum_{s'\in \mathcal{S}}\mathcal{Z}(s',o,a) \sum_{s\in \mathcal{S}} \mathcal{T}(s,a,s')b(s).
\end{align}

The optimal value function, denoted $V_d^*(b)$, represents the maximum discounted expected reward obtainable from belief $b$ with $d$ steps remaining. It is defined recursively by the Bellman optimality equation:
\begin{align}
  V_d^*(b) &= \max_{a \in \mathcal{A}}
    \Bigl\{ r(b,a) 
      + \gamma \mathbb{E}[V_{d-1}^*(\tau(b,a,o))] \Bigr\}, \label{eq:optimal_value}
\end{align}
with $V_0^*(b) = 0$. Here, the expected immediate reward is $r(b,a) = \sum_{s\in\mathcal{S}} b(s)\,\mathcal{R}(s,a)$, and the expectation over future values is computed with respect to the observation probabilities:
\begin{align}
    \mathbb{E}[\cdot] = \sum_{o \in \mathcal{O}} \mathbb{P}(o \mid b,a)\, V_{d-1}^*(\tau(b,a,o)).
\end{align}

It is often useful to express the optimal value for a particular action at a belief using the $Q$-value function:
\begin{align}
  Q_d^*(b,a) &= r(b,a)
      + \gamma \mathbb{E}[V_{d-1}^*(\tau(b,a,o))].
\end{align}

Consequently, the optimal decision rule for the horizon step $d$ can be retrieved directly from the $Q$-values:
\begin{equation}
  \pi_d^*(b) \in \operatorname*{arg\,max}_{a\in \mathcal{A}}\;Q_d^*(b,a).
\end{equation}
\section{Value of Information for POMDPs}
The \textit{curse of history} renders exact planning intractable for non-trivial horizons and problem sizes. However, the standard formulation implicitly assumes that every observation is critical for decision-making, necessitating reasoning over all possible reachable histories. In practice, the utility of observation information is rarely uniform across the belief space; there are often regions where an agent can act effectively without immediate observation feedback. In this section, we introduce our adaptive value of information (VOI) reasoning framework to exploit this non-uniformity. By formalizing the trade-off between open-loop execution (disregarding observations) and closed-loop execution (branching on observations), we expose a principled mechanism that enables an agent to selectively prune the history space while maintaining bounds on optimality. This lays the fundamental groundwork for our algorithm that exploits this tradeoff for improved search efficiency.

\subsection{Adaptive Value of Information}

As discussed in the previous section, the optimal solution to a POMDP satisfies the Bellman equation given in \eqref{eq:optimal_value}. Here, the agent performs a Bayesian belief update \eqref{eq:belief_update_full} after every action and observation pair. An agent that integrates feedback at every step is said to act in a \emph{closed-loop} manner: by acting, observing, and updating its belief, it closes the feedback loop prior to every subsequent decision. The corresponding closed-loop value function is
\begin{eqnarray}
        V_d^{CL}(b) = \max_{a \in \mathcal{A}} 
        \Bigl\{ r(b,a) + \gamma \mathbb{E}_{o}\bigl[V_{d-1}^{CL}(\tau(b,a,o)) \bigr] \Bigr\}.
        \label{eq:vcl}
\end{eqnarray}

From a computational complexity perspective, the number of distinct histories at depth $d$ is
\begin{align}
    |\mathcal{H}_d|^{CL} = (|\mathcal{A}|\cdot |\mathcal{O}|)^d.
\end{align}

Since the branching factor depends on the size of the observation space, the number of histories grows exponentially with the planning horizon. This relationship directly highlights the \emph{curse of history}.

Alternatively, an agent may choose to ignore observations, planning based purely on expected beliefs. This is referred to as \emph{open-loop} reasoning. Formally, an agent employs an \emph{open-loop belief update} by marginalizing over the observation space in \eqref{eq:belief_update_full}, such that
\begin{eqnarray}
    \tau(b,a)(s') = \mathbb{E}_o[\tau(b,a,o)] = \sum_{s} \mathcal{T}(s,a,s')\, b(s). \label{eq:belief_update_open}
\end{eqnarray}
The resulting \emph{open-loop value function} is defined as
\begin{eqnarray}
        V_d^{OL}(b) = \max_{a \in \mathcal{A}} 
        \Bigl\{ r(b,a) + \gamma V_{d-1}^{OL}(\tau(b,a)) \Bigr\},
        \label{eq:vol}
\end{eqnarray}
with $V^{OL}_0(b) = 0$. The computational advantage is immediately apparent: the number of distinct histories at depth $d$ when acting open-loop drops by a factor of $|\mathcal{O}|$ to
\begin{align}
    |\mathcal{H}_{d}|^{OL} = |\mathcal{A}|^d \ll|\mathcal{H}_{d}|^{CL}.
\end{align}

However, acting open-loop is generally suboptimal. Since the value function is convex with respect to the belief state \citep{smallwood1973optimal}, Jensen's inequality implies that the value of the expected belief is upper bounded by the expected value of the posterior beliefs:
\begin{align}
        \mathbb{E}_o[V_d^*(\tau(b,a,o))] &\geq V_d^*(\mathbb{E}_o[\tau(b,a,o)]),\\
        V^{CL}_d &\geq V^{OL}_d.
\end{align}

We can quantify this suboptimality as the simple \emph{value of information} (VOI), defined as the difference between the closed-loop and open-loop values \citep{wei2024value}:
\begin{eqnarray}
    VOI_d(b) & = & V_{d}^{CL}(b) - V_d^{OL}(b).
    \label{eq:voi_simple}
\end{eqnarray}

The simple VOI presents a static dichotomy: it compares a policy that \emph{always} observes against one that \emph{never} observes over the horizon $d$. However, this fails to capture the potential for adaptive reasoning over the value of information. In many POMDPs, the performance gain from reasoning about observations varies over the belief space. For instance, an agent may justifiably act in an open-loop manner while observations are uninformative, but must switch to closed-loop planning as uncertainty accumulates. Therefore, a more nuanced treatment of VOI requires evaluating the value of information locally at each belief state, allowing the agent to dynamically switch between open-loop and closed-loop modalities.

In what follows, we outline a framework designed to enable this selective reasoning. We begin by defining a new value function, denoted as $\hat{V}_d^*$, which is the value of a policy that selects between open-loop and closed-loop reasoning strategies based on the criterion defined below.

The core of this framework is our proposed \emph{adaptive value of information} criterion. At each decision step, the agent computes the value difference between the two strategies and acts closed-loop only if the value gain exceeds a threshold parameter $\kappa$. We define the adaptive VOI as:
\begin{eqnarray}
    \hat{VOI}^\kappa_d(b) & = & \hat V_{d}^{CL}(b) - \hat V_{d}^{OL}(b),
    \label{eq:voi_adaptive}
\end{eqnarray}
where 
\begin{align}
        \hat{V}_{d}^{OL}(b) & = \max_{a \in \mathcal{A}} 
        \Bigl\{ r(b,a) + \gamma \hat{V}_{d-1}^*(\tau(b,a))\Bigr\},
        \label{eq:vol_adaptive} \\[2mm]
        \hat{V}_{d}^{CL}(b) & = \max_{a \in \mathcal{A}} 
        \Bigl\{ r(b,a) + \gamma \mathbb{E}_{o}\bigl[\hat{V}_{d-1}^*(\tau(b,a,o))\bigr]\Bigr\}.
        \label{eq:vcl_adaptive}
\end{align}

Crucially, unlike the simple VOI, where open-loop and closed-loop values are computed independently, our components $\hat{V}_d^{OL}$ and $\hat{V}_d^{CL}$ are coupled recursively through the optimal $\kappa\mhyphen$adaptive value function $\hat{V}_{d-1}$, defined as
\begin{align}
    \hat{V}_{d}^*(b)=
    \begin{cases}
        \hat{V}_d^{OL}(b), & \text{if }\; \hat{V}_d^{OL}(b)\ge \,\hat{V}_d^{CL}(b) - \kappa \lvert \hat{V}_d^{CL}(b)\rvert,\\[2mm]
        \hat{V}_d^{CL}(b), & \text{otherwise,}
    \end{cases}
    \label{eq:optimal_value_cond}
\end{align}
for $\kappa\in[0,1]$.

With $\hat V^*_0=0$, it is clear by induction on $d$ that $\hat{V}^*_d$ is bounded and exists for any $d \in \{0, \dots, D\}$. \eqref{eq:optimal_value_cond} represents a Bellman-like update under the operator $\mathbb{B}_\kappa$, where
\begin{eqnarray}
    \hat{V}^*_{d}(b)\;=\; \mathbb{B}_\kappa \hat{V}^*_{d-1}(b).
\end{eqnarray}

$\hat{V}^*$ defines a decision rule at each belief $b$ based on $\kappa$. If the open-loop backup value at $b$ is within $\kappa$ of the closed-loop value, the $\hat{VOI}^\kappa_d(b)$ is sufficiently \emph{low}. Therefore, the agent performs the open-loop backup, which ignores redundant information. On the other hand, if the $VOI^\kappa_d(b)$ is relatively \emph{high}, there is too much value to be lost by disregarding observations; therefore, the agent performs the full closed-loop backup. This adaptive reasoning is applied at every decision step, enabling the agent to \emph{selectively} incur the cost of closed-loop reasoning only when the adaptive VOI justifies it. As a result, the policy interpolates between purely open-loop behavior (when $\kappa$ is large and VOI is generally low) and fully closed-loop behavior (when $\kappa$ is small and VOI is high), reducing unnecessary reasoning over future observations in low-VOI regions while preserving near-optimal performance where observations are valuable.

\subsection{Bounded suboptimality}

Here, we analyze the suboptimality of the adaptive VOI value. We define the regret $\rho$ of the adaptive VOI value for $d$ steps at a belief $b$ as
\begin{eqnarray}
        \rho_d(b) = \left| V_d^*(b) - \hat{V}_d^*(b) \right|.
        \label{eq:regret}
\end{eqnarray}

In the following theorem, we prove that the worst-case regret is bounded by a function of $\kappa$. 

\begin{theorem}[Bounded Regret]
    \label{thm:suboptimality}
    Let $\kappa\in[0,1]$. Then, for any $b$ and $d \geq 1$,
    \begin{align*}
        \rho_{d}(b) \leq \kappa \frac{R_{max}}{1-\gamma} \left(\frac{1-\gamma^d}{1-\gamma}\right).
    \end{align*}
\end{theorem}

\begin{proof}
Let $\rho_d(b)= \left| V_d^*(b) - \hat V_d^*(b)\right|$. Let $\mathbb{B}$ be the Bellman operator for \eqref{eq:optimal_value}. Then, by the triangle inequality,
\begin{align*}
    \rho_d(&b)  = \left| \mathbb{B}V_{d-1}^*(b) - \mathbb{B}_\kappa \hat V_{d-1}^*(b) \right| \\
    & \leq \left| \mathbb{B}V_{d-1}^*(b) - \mathbb{B} \hat V_{d-1}^*(b) \right| + \left| \mathbb{B}\hat V_{d-1}^*(b) - \mathbb{B}_\kappa \hat V_{d-1}^*(b) \right|
\end{align*}

The first term is a $\gamma\mhyphen$contraction. For the second term, \eqref{eq:optimal_value_cond} implies the single-step suboptimality is at most $\kappa\,\lvert \hat V^{CL}_d(b) \rvert \leq \kappa\,\frac{R_{max}}{1-\gamma}$, and since $V^*_0 = \hat{V}^*_0$, we have that
\begin{align*}
    \rho_d(b) \leq \sum_{t=1}^{d} \gamma^{d-t} (\kappa \lvert \hat V^{CL}_d(b) \rvert)
    &\leq \sum_{t=1}^{d} \gamma^{d-t} (\kappa\,\frac{R_{max}}{1-\gamma})\\
    &= \kappa \frac{R_{max}}{1-\gamma} \left(\frac{1-\gamma^d}{1-\gamma}\right)
\end{align*}
\end{proof}


Theorem~\ref{thm:suboptimality} shows that the adaptive VOI reasoning provides a sound and bounded approximation of $V_d^*$. In particular, the adaptive VOI regret scales linearly with the parameter $\kappa$, ensuring that the approximation error remains controlled. This result motivates the use of the adaptive VOI to efficiently allocate planning computation in the pursuit of mitigating the curse of history, especially in settings where the intrinsic VOI is low and the resulting regret is small.

\section{POMDP Planning using Value of Information}
\label{sec:voipomdp}
This section demonstrates how the adaptive VOI reasoning framework can be integrated into a planning algorithm to effectively allocate computational effort. Optimizing the adaptive value function $\hat{V}^*$ directly would traditionally require evaluating and comparing the open-loop \eqref{eq:vol_adaptive} and closed-loop \eqref{eq:vcl_adaptive} value functions at every step. Rather than performing this comparison as an external step, we introduce the \textit{Value of Information POMDP} (VOI-POMDP). This meta-level representation explicitly encodes the choice between open-loop and closed-loop execution modes as distinct actions within the problem structure. By unifying these modalities into a single decision space, we enable standard solvers to automatically perform VOI reasoning during the search process. Subsequently, we propose Value of Information Monte Carlo Planning (VOIMCP), a solver based on Monte Carlo Tree Search (MCTS) that exploits the VOI-POMDP structure for efficient planning.

\subsection{VOI-POMDP Representation}
The adaptive VOI framework establishes a mathematical basis for disregarding observations. Consequently, computing $\hat{V}^*$ can be viewed as a meta-level choice between open-loop and closed-loop execution modalities at each belief. The following POMDP transformation encodes this meta-level choice as distinct actions directly into the problem structure.

Given a POMDP $\mathcal{P} = (\mathcal{S}, \mathcal{A}, \mathcal{O}, \mathcal{T}, \mathcal{R}, \mathcal{Z}, D, b_0,  \gamma)$, we define the VOI-POMDP as the tuple $\mathcal{P'} = (\mathcal{S}, \mathcal{A'}, \mathcal{O'}, \mathcal{T}', \mathcal{R}', \mathcal{Z'}, D, b_0,  \gamma)$. The state space $\mathcal{S}$, initial belief $b_0$, horizon $D$, and discount factor $\gamma$ remain identical to the original problem. We construct the augmented components as follows. Let $\mathcal{A}$ be the original action space. We duplicate each action and define an augmented action space with an open-loop action set and a closed-loop action set:
\begin{equation*}
    \begin{aligned}
      &\quad \mathcal{A}_{\mathrm{OL}} 
          = \{\,a_{\mathrm{OL}}\mid a\in\mathcal{A}\},\quad
          \mathcal{A}_{\mathrm{CL}} 
          = \{\,a_{\mathrm{CL}}\mid a\in\mathcal{A}\},\\
      &\quad \mathcal{A}' 
          = \mathcal{A}_{\mathrm{OL}}\cup\mathcal{A}_{\mathrm{CL}}
          = \{(a,m)\mid a\in\mathcal{A},\,m\in\{\mathrm{OL},\mathrm{CL}\}\}.\\
    \end{aligned}
\end{equation*}

Further, we augment the observation space with a \textit{null} (or uninformative) observation
\begin{equation*}
    \mathcal{O'} := \mathcal{O} \cup \{o_{\mathrm{null}}\},
\end{equation*}
where $o_{null}$ is received with probability $1$ under open-loop actions while observations under closed-loop actions behave as in the original POMDP:
\begin{equation}
    \mathcal Z'(s',o, a_m) =
    \begin{cases}
        \mathbf{1}\{o = o_{\mathrm{null}}\}, & m = \mathrm{OL}, \\[6pt]
        Z(s',o, a), & m = \mathrm{CL}.
    \end{cases}
\end{equation}
\noindent $\mathcal{T}'$ and $\mathcal{R}'$ map both open-loop and closed-loop action variants to the same values:
\begin{align}
    \mathcal{T}'(s,a_m,s')&=\mathcal{T}(s,a,s'), \\
    \mathcal{R}'(s,a_m)&=\mathcal{R}(s,a).
\end{align}

\noindent The belief update for a VOI-POMDP depends on the action taken. For open-loop actions, the computed belief update marginalizes over observations, as in \eqref{eq:belief_update_open}. Closed-loop actions allow for the full belief update as in \eqref{eq:belief_update_full}. Thus,
\begin{align}
    b'(s') = 
    \begin{cases}
        \tau(b,a_m,o_{\mathrm{null}}) = \sum_{s} \mathcal T'(s,a,s')\, b(s), & m = \mathrm{OL}, \\
        \tau(b,a_m,o) = \tau(b,a,o)  , & m = \mathrm{CL}.
    \end{cases}
\end{align}

The adaptive VOI value function $\hat{V}'_d$ for $\mathcal{P}'$ is computed recursively for $d \in \{0, \dots, D\}$, starting from the base case $\hat{V}'_0 \equiv \hat{V}_0 =  0$. To simplify notation, we denote the optimal value function for $\mathcal{P}'$ as $\hat{V}'_d$, dropping the superscript $*$ (i.e., $\hat{V}'_d = \hat{V}_d'^*$). The value backup is defined as:

\begin{equation}
    \label{eq:voipomdp_value_backup}
    \hat{V}_d'(b) = Q'_d(b, \hat{\pi}_d^*(b)), \quad \forall d \in \{0, \dots, D\},
\end{equation}
where $Q'_d$ is the standard action-value for $\mathcal{P}'$, computed as:
\begin{equation}
    Q'_d(b, a) = r(b, a) + \gamma \mathbb{E} \left[ \hat{V}'_{d-1}(\tau(b, a, o)) \right].
\end{equation}
and $\hat{\pi}_d^*$ denotes the optimal adaptive VOI-based policy at depth $d$ that satisfies for every belief $b$

\begin{equation}
    \hat{\pi}_d^*(b) = \arg\max_{a\in\mathcal{A}'} 
    \begin{cases} 
    Q'_d(b, a), & a \in \mathcal{A}_{OL}, \\
    Q'_d(b, a) - \kappa |Q'_d(b, a)|, & a \in \mathcal{A}_{CL}.
    \end{cases}
\end{equation}

We see that the VOI-POMDP representation preserves the same (optimal) adaptive VOI value function as \eqref{eq:optimal_value_cond}.
\begin{proposition}
    \label{prop: pprime optimal value}
    Let $\mathcal P$ be the original POMDP and let $\mathcal P'$ be the corresponding transformed VOI-POMDP defined above. We have that for every depth $d\in\{0,\dots,D\}$ and belief $b\in \Delta(\mathcal{S})$,
    \begin{align}
        \hat V_{d}'(b) = \hat V_d^*(b).
    \end{align}
\end{proposition}

The VOI-POMDP representation has distinct benefits in planning. By mapping the meta-level choice to the action and value backup space, we allow standard solvers to optimize for $\hat{V}^*$ using minor modifications to existing search heuristics. The representation structurally simplifies the search process. Although duplicating the action space doubles the nominal action branching ($|\mathcal{A}'| = 2|\mathcal{A}|$) and raises the worst-case branching factor from $O((|\mathcal{A}||\mathcal{O}|)^{d})$ to $O((|\mathcal{A}||\mathcal{O}|+|\mathcal{A}|)^{d})$, this cost is offset by the collapse of the observation space under open-loop actions. Under open-loop actions, the action-observation branching factor drops to $1$. In practice, this prevents expanding dense observation branches of a search tree whenever VOI is low, significantly reducing the effective branching factor and enabling deeper planning under the same computational budget.

\subsection{Value of Information Monte Carlo Planning}
To leverage the advantages of VOI reasoning for POMDP planning, we build on PO-UCT \cite{silver2010pomcp}, a widely used Monte Carlo tree search (MCTS) algorithm for large POMDPs. PO-UCT extends the upper confidence trees (UCT) algorithm to partially observable problems by building a search tree over the space of histories instead of states, relying solely on a generative model $\mathcal{G}$ of the problem.

Our algorithm, called Value of Information Monte Carlo Planning (VOIMCP), models adaptive VOI reasoning by solving for $\hat{V}^*_d$ on $\mathcal{P}'$. The procedure for VOIMCP is outlined in Algorithm \ref{alg:voi-mcp}. The global parameters include the number of tree queries $n$, the exploration bonus $\beta$, the problem horizon $D$, the discount factor $\gamma$, and $\kappa$. The $\textsc{Search}$ procedure serves as the entry point, taking as input a belief $b$ and repeatedly calling the $\textsc{Simulate}$ procedure to build out the tree $T$. At each history-action node $T(ha)$ in the tree, the visitation count $N(ha)$, and the estimated value $\bar Q(ha)$ are stored. These statistics are all initialized to $0$. While our theoretical VOIMCP formulation has every simulation iteration reaching the full depth $D$, our practical implementation follows the standard PO-UCT approach for efficiency (see attached Supplementary). Specifically, when the \textsc{Simulate} procedure expands a new history node $h \notin T$, it terminates the iteration and invokes a $\textsc{ValueEstimate}$ procedure to estimate the value of initialized leaf nodes. A common evaluation approach is to simulate with a heuristic rollout policy.

We adopt the polynomial upper confidence bound (UCB) from \citet{shah2022polyuct}. Our primary contribution is a modified polynomial UCB, which implements the adaptive VOI backup in \eqref{eq:voipomdp_value_backup}. We approximate the adaptive VOI backup operator by replacing the exact values $Q'_d$ with Monte Carlo estimates $\bar{Q}$ and applying a penalty $\kappa$ directly to the closed-loop exploitation term. For open-loop actions, the selection metric follows the standard optimistic principle:
\begin{eqnarray}
    UCB_{VOI}(ha_{OL}) = \bar{Q}(ha_{OL}) + B_N(ha_{OL}),
    \label{eq:voiucb_ol}
\end{eqnarray}
where $\bar{Q}(\cdot)$ denotes the mean value estimate of simulations that visit a history-action pair, and $B_N(\cdot)$ denotes the polynomial bonus term proposed by \citet{shah2022polyuct}:
\begin{align}
    B_{N}(ha_m) = \beta^{1/\xi} \cdot \frac{ N(h)^{\alpha/\xi}}{N(ha_{m})^{1-\eta}}, \quad\forall m\in\{OL,CL\},
\end{align}
parameterized by $\beta, \xi, \alpha$, and $\eta$.

For closed-loop actions, we leverage the adaptive VOI objective to enforce a $\kappa$-dependent bias against unnecessary observation branching. This is achieved by deflating the exploitation term in the UCB calculation:
\begin{equation}
    \begin{aligned}
        UCB_{VOI}(ha_{CL}) = & \;\bar{Q}(ha_{CL})\,- \,\kappa \lvert \bar{Q}(ha_{CL}) \rvert \\ 
         & + B_N(ha_{CL}).
    \end{aligned}
    \label{eq:voiucb_cl}
\end{equation}

This deflation biases the search towards open-loop actions whenever the potential VOI is insufficient to meet relative tolerance $\kappa$, effectively steering the search toward the computationally cheaper open-loop branches. By integrating this VOI-guided reasoning into our selection policy, VOIMCP focuses computation on actions and observations that are likely to yield significant value, enabling deeper tree search while maintaining near-optimal policy quality.

\subsubsection{Annealing for Global Convergence} While a fixed $\kappa$ prunes unnecessary observation branching based on a constant threshold (converging to $\hat{V}_D^*$), recovering the optimal POMDP value $V^*_D$ requires this bias to vanish over time. To achieve this, we employ an annealing schedule where $\kappa_N$ decays as a function of the visitation counts. Specifically, we require $\kappa$ to decay as a function of the confidence width, i.e., $0 \leq \kappa_{i}(t, s_i) \leq \frac{1}{c_\kappa R_{max}}B_{t, s_i}$ where $c_\kappa>1$. This ensures that the selection bias is always dominated by the confidence width; as $B_N \to 0$, the bias vanishes and the algorithm recovers the optimal value function in the limit.

\begin{algorithm}[t]
\caption{VOIMCP}
\label{alg:voi-mcp}
\small
\vspace{0pt}
\begin{algorithmic}[1]
\Procedure{Search}{$b$}
  \ForAll{$i = 1,\dots,n$}
    \State $s\sim b$
    \State $h_0 = \{b\}$
    \State \Call{Simulate}{$s,h_0,0$}
  \EndFor
  \State \Return $\displaystyle
    \arg\max_{a\in\mathcal{A}'}\bar Q(h_0a)$   
  
\EndProcedure
\vspace{2pt}
\Procedure{Simulate}{$s,h,depth$}
  \If{$depth=D$} \State \Return $0$ \EndIf
  \If{$h\notin T$}
    \ForAll{$a\in\mathcal{A'}$}
      \State $T(ha)\leftarrow(0,0)$
    \EndFor
  \EndIf
  \State $a^*\!\in\underset{a\in\mathcal{A'}}{\!\arg\max}\; UCB_{VOI}(ha)$
  \If{$a^* \in \mathcal{A}_{CL}$}
  \State $(s',o,r)\sim \mathcal{G}(s,a^*)$
  \Else
  \State $(s',r)\sim \mathcal{G}(s,a^*)$
  \State $o\leftarrow o_{null}$ 
  \EndIf
  \State $N(h)\leftarrow N(h)+1$; $N(ha^*)\leftarrow N(ha^*)+1$
  \State $R\leftarrow r+\gamma\Call{Simulate}{s',ha^*o,depth{+}1}$
  \State $\bar Q(ha^*)\leftarrow \bar Q(ha^*)+\dfrac{R - \bar Q(ha^*)}{N(ha^*)}$
  \State \Return $R$
\EndProcedure
\end{algorithmic}
\end{algorithm}

\subsection{Convergence Analysis of VOIMCP}

Here, we analyze the non-asymptotic properties of VOIMCP. First, we prove that for a fixed $\kappa$, VOIMCP converges polynomially to the optimal VOI-value function $\hat{V}_D^*$. Let $\bar{V}_n(b_0)=\frac{1}{n}\sum_{a\in\mathcal{A}'}N(h_0a)\bar Q(h_0a)$ denote the value estimate at the root after $n$ simulations, where $h_0=\{b_0\}$. Then we have the following result.

\begin{theorem}
\label{thm:VOIMCP convergence}
Consider a POMDP $\mathcal{P}$ with $D \geq 1$ and a fixed $\kappa \in [0,1]$. There exists a valid configuration of the algorithm's depth dependent parameters (specifically, depth dependent $\xi, \alpha, \beta$ detailed in the Supplementary) such that for any $\eta \in [0.5, 1)$ and any initial belief $b_0$, the following claim holds for the output $\bar{V}_n(b_0)$ of VOIMCP with $n$ simulations:
\begin{align}
    |E[\bar{V}_n(b_0)] - \hat{V}_D^*(b_0)| \leq O(n^{\eta-1}).
\end{align}
\end{theorem}

\begin{proof}[Proof Sketch]
    We defer the detailed derivation to the Supplementary. The proof structure adapts the analysis of \cite{shah2022polyuct} and models the search tree as a hierarchy of non-stationary $\kappa\mhyphen$biased Multi-Armed Bandit (MAB) problems. Using backward induction from the leaf nodes to the root, we prove polynomial concentration and convergence to the VOI-optimal value. Finally, we transfer this result to VOIMCP via a history-MDP equivalence.
\end{proof}

We now demonstrate that by annealing $\kappa$ as a function of the confidence bonus, the selection bias vanishes sufficiently fast to guarantee convergence to $V_D^*$.

\begin{theorem}
    \label{thm: voimcp global convergence}
   Consider a POMDP $\mathcal{P}$ with $D \geq 1$ and $0 \leq \kappa_{i}(t, s_i) \leq \frac{1}{c_\kappa R_{max}}B_{t, s_i}$ for some $c_\kappa > 1$. There exists a valid configuration of the algorithm's depth dependent parameters (specifically, depth dependent $\xi, \alpha, \beta$ detailed in the Supplementary) such that for any $\eta \in [0.5, 1)$ and any initial belief $b_0$, the following claim holds for the output $\bar{V}_n(b_0)$ of VOIMCP with $n$ simulations: 
   \begin{align}
       |E[\bar{V}_n(b_0)] - V_{D}^*(b_0)| \leq O(n^{\eta-1}).
   \end{align}
\end{theorem}

\begin{proof}[Proof Sketch]
    We defer the detailed derivation to the Supplementary. The proof follows the same backward induction structure established in Theorem \ref{thm:VOIMCP convergence}. The distinction lies in the selection bias for closed-loop actions: by enforcing $\kappa_{i}(t, s_i) \leq \frac{1}{c_\kappa R_{max}}B_{t, s_i}$, the VOI penalty effectively acts as a vanishing exploration noise rather than a fixed bias. For every node, as visitation count $\to \infty$, this bias term shrinks to zero and is asymptotically dominated by the suboptimality gap of the actions. Propagating this result from leaves to the root, we demonstrate that the value estimates at each depth concentrate around the true optimal value function $V^*$.
\end{proof}
\begin{figure*}[ht]
  \centering
  \begin{subfigure}[b]{0.33\textwidth}
    \includegraphics[width=\linewidth]{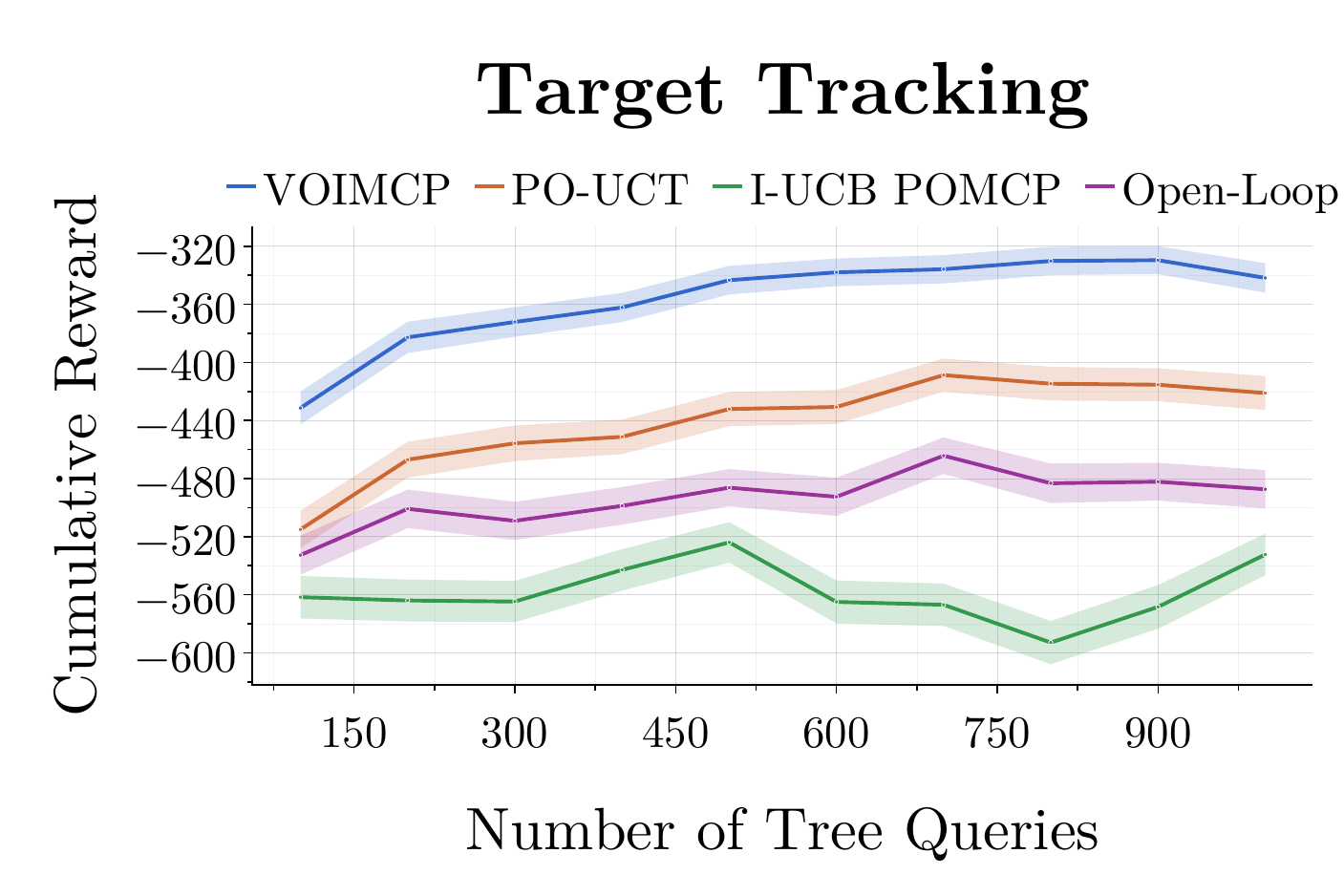}
    \end{subfigure}
  \begin{subfigure}[b]{0.33\textwidth}
    \includegraphics[width=\linewidth]{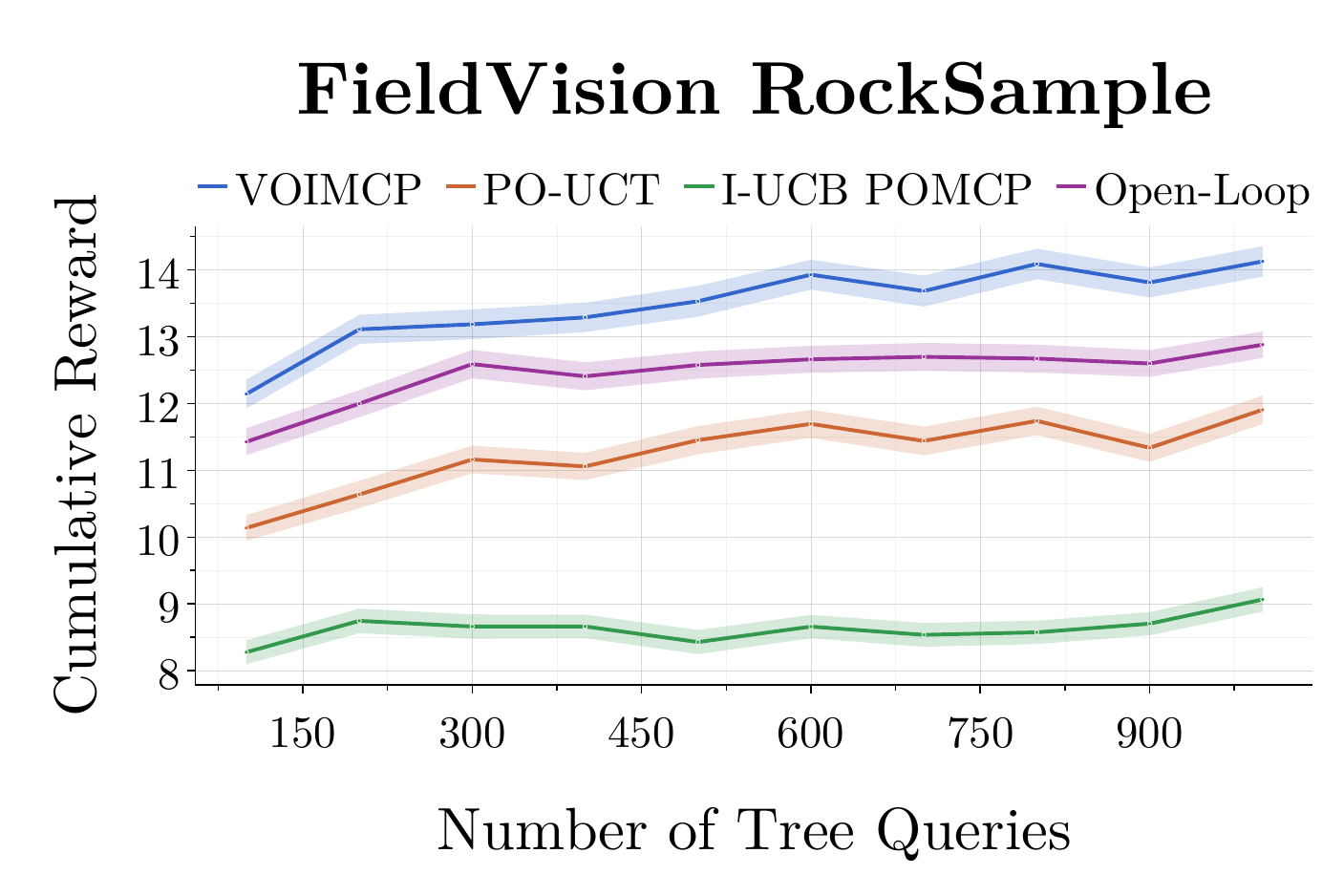}
    \end{subfigure}
  \begin{subfigure}[b]{0.33\textwidth}
    \includegraphics[width=\linewidth]{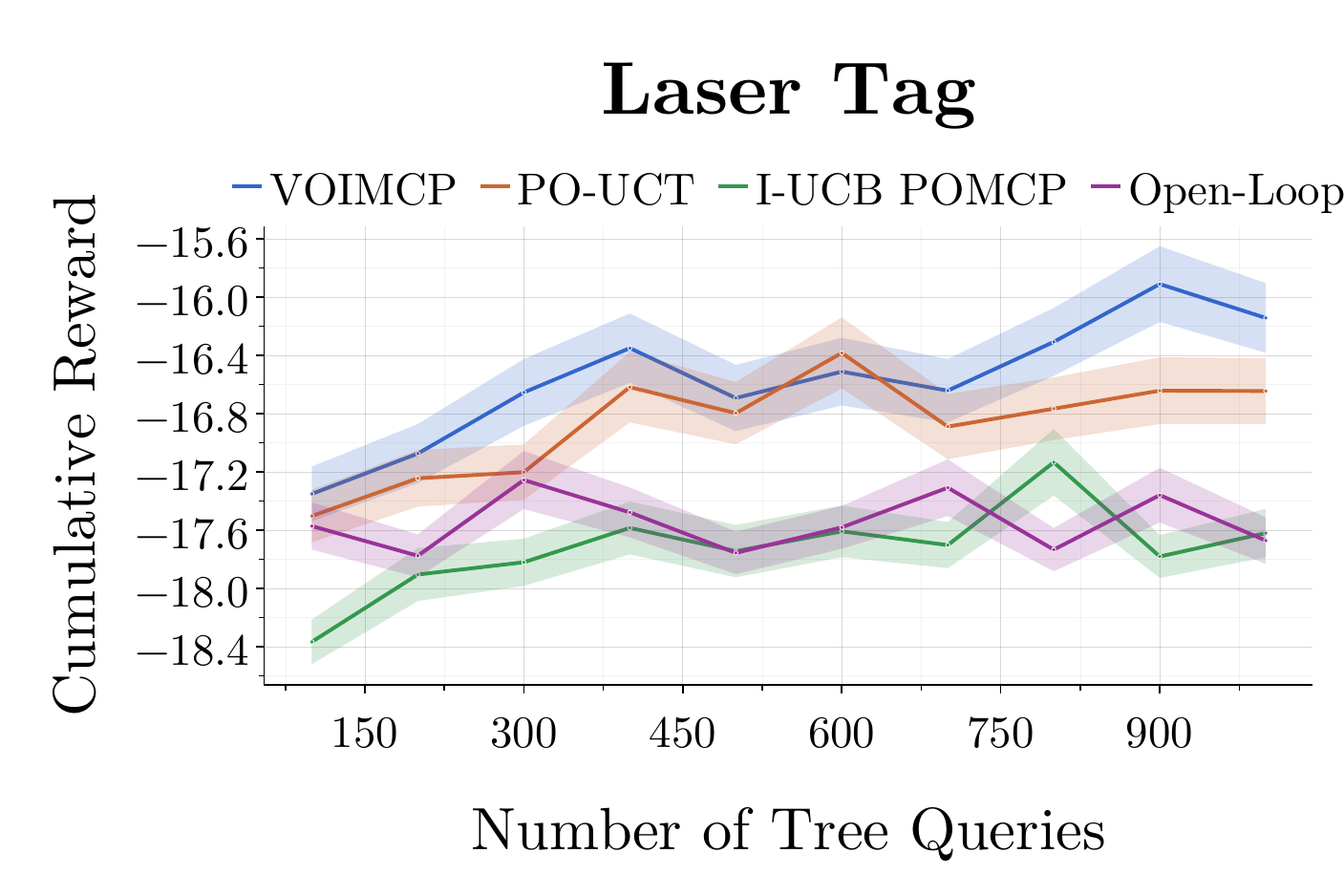}
    \end{subfigure}
  \caption{Comparative benchmark results presenting discounted cumulative reward against the number of tree queries over $1000$ trials. The lighter colored ribbons around the Monte Carlo mean display the $95\%$ confidence interval.}
  \label{fig:benchmark results}
\end{figure*}

\section{Empirical Evaluation}
In this section, we evaluate our adaptive VOI approach.

\subsubsection{Domains}
We evaluate VOIMCP with a fixed $\kappa$\footnote{\label{fn:supp2}Additional results for annealing experiments are available in the Supplementary.} on three POMDP benchmark domains with large observation spaces:
\begin{enumerate}
    \item \textbf{Target Tracking} \cite{flaspohler2020belief}: An agent aims to track a moving target on a $10\times10$ grid. The agent’s own state is fully observable, while the target is only partially observable through noisy measurements.
    \item \textbf{FieldVision RockSample} \cite{ross2008online}: A variant of RockSample with a much larger observation space. A robot explores a grid to collect valuable rocks while avoiding bad rocks. Rock quality is observed noisily through sensor readings after each action.
    \item \textbf{Laser Tag} \cite{despot}: A robot aims to tag a moving target on a 7$\times$11 grid with eight randomly placed obstacles. The robot receives noisy distance measurements to the nearest obstacle in each direction, yielding an observation space of roughly $1.5\times10^6$ observations.
\end{enumerate}

\subsubsection{Comparison Algorithms}
 We evaluate VOIMCP against the following baselines:
\begin{enumerate}
    \item \textbf{PO-UCT} \cite{silver2010pomcp}: The standard MCTS solver for POMDPs, which uses the original UCB1 selection policy. This also serves as an ablation study, as VOIMCP reduces to PO-UCT if the meta-level action space is removed and standard selection policy is used.
    \item \textbf{I-UCB POMCP} \cite{alves2023ibpomcp}: A variant of POMCP that augments the UCB term with an observation entropy heuristic.
    \item \textbf{Open-Loop}: A VOIMCP variant that acts purely in an open-loop manner by only expanding null observations. This baseline essentially plans a sequence of actions without reasoning about future observations.
\end{enumerate}

To isolate the impact of the search component in the planning algorithms, we standardize the belief update mechanism across all methods by using a Sequential Importance Resampling (SIR) particle filter \citep{Gordon1993sir}. Similarly, we estimate leaf node values using a random rollout policy for all approaches. 

\subsubsection{Implementation Details}
We tuned the key parameters of each approach and report the best performance. The horizon $D$ for each algorithm was optimized from the set $\{20,40,60,80\}$, and the exploration constant $c$ for POMCP and VOIMCP was optimized from the set $\{1,10,100,1000\}$. For VOIMCP, $\beta^{1/\xi}=c$ and $\eta=\frac{1}{2}$ for all nodes in the tree, resulting in a bonus term $B_N$ that scales as $N(h)^\frac{1}{4} / \sqrt{N(ha)}$, consistent with \citet{shah2022polyuct}. For I-UCB POMCP, we adopt the parameters recommended by \citet{alves2023ibpomcp}. Detailed parameters used for the experiments are outlined in the supplementary material. We evaluate all approaches under computational budgets ranging from $100$ to $1000$ tree queries, averaged over $1000$ trials. All algorithms were implemented in Julia and executed single-threaded on a computer equipped with an Intel Xeon(R) W-1370P processor.

\begin{figure*}[ht]
    \centering
    \begin{subfigure}[b]{0.33\textwidth}
        \includegraphics[width=\linewidth]{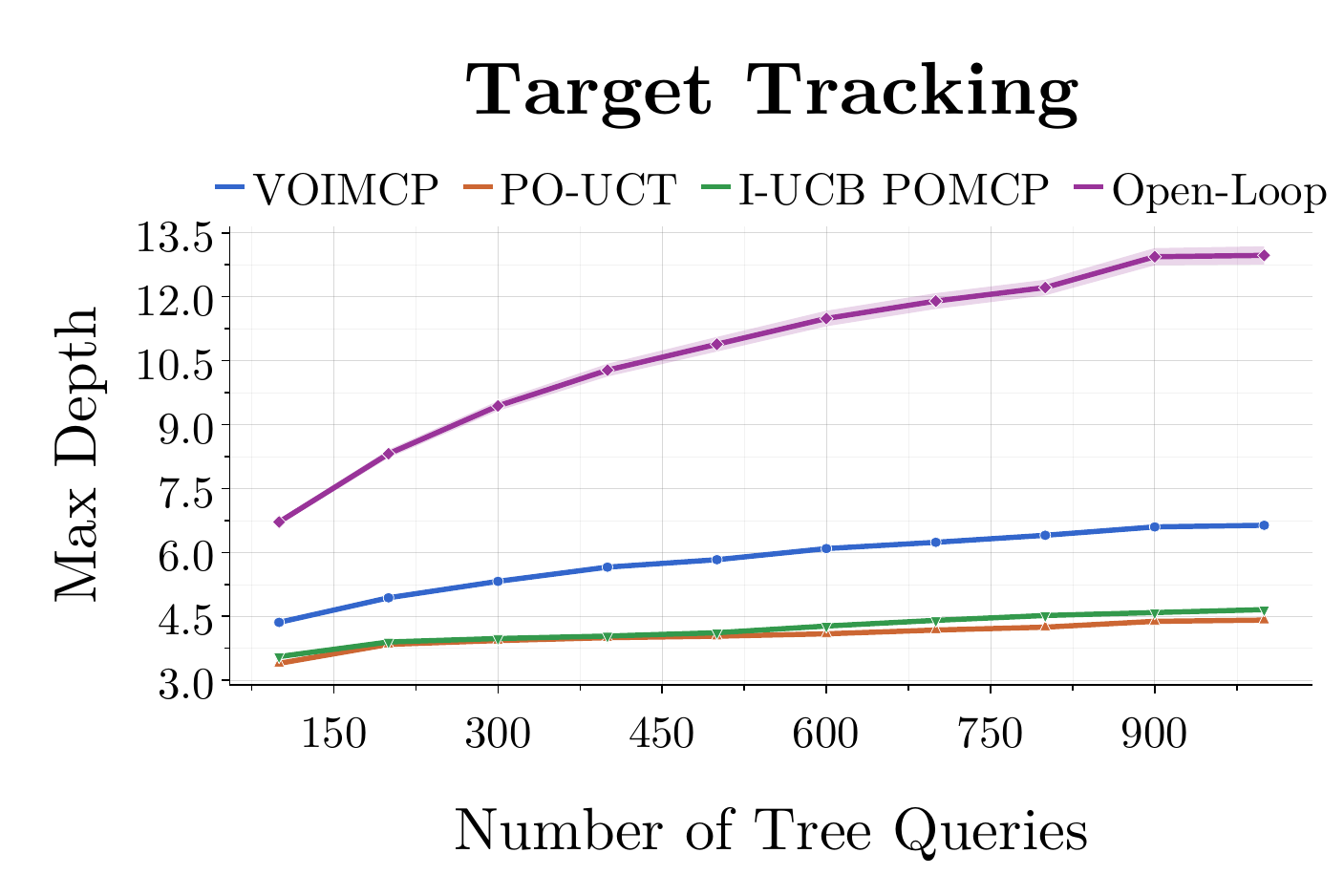}
        \end{subfigure}
    \begin{subfigure}[b]{0.33\textwidth}
        \includegraphics[width=\linewidth]{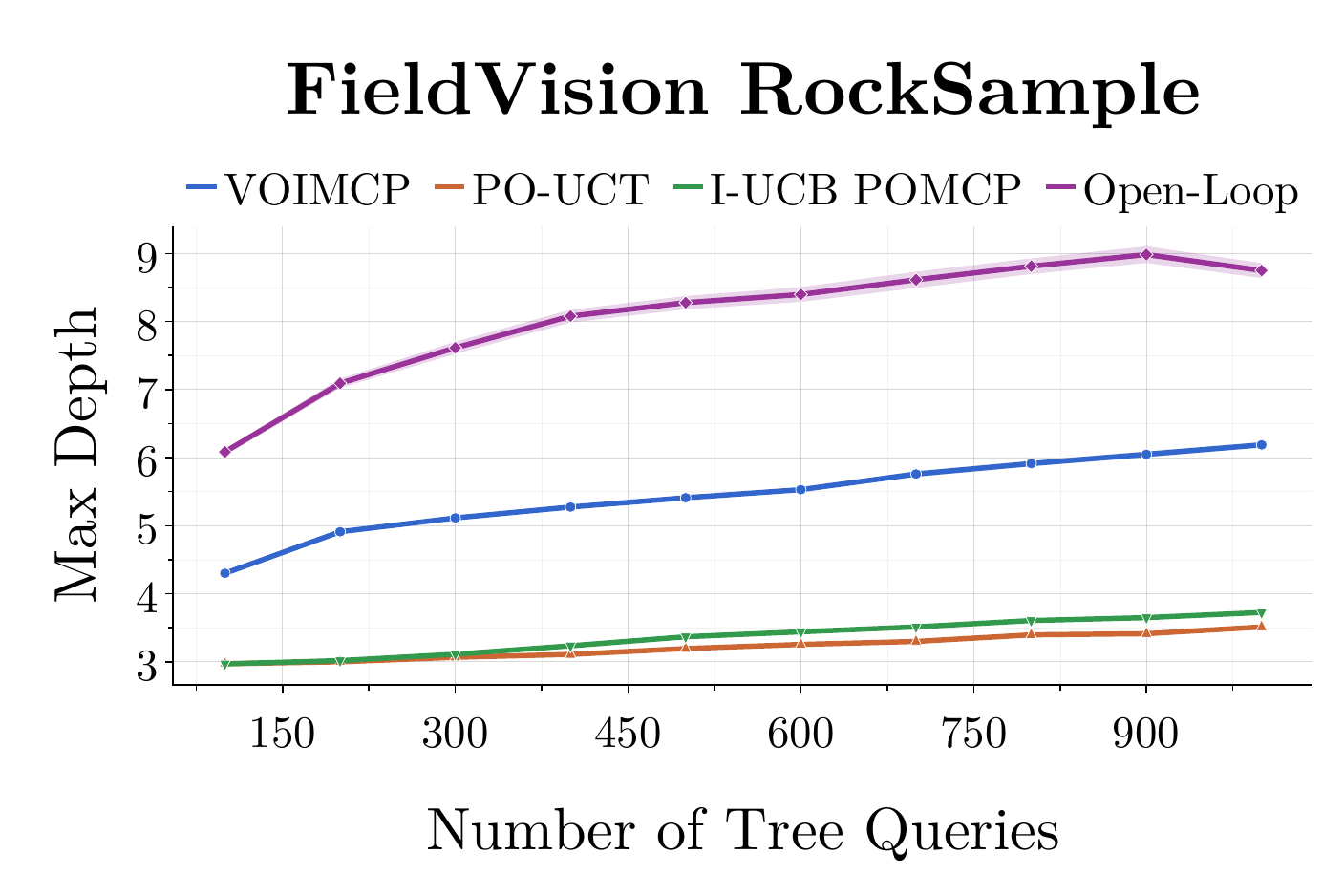}
        \end{subfigure}
    \begin{subfigure}[b]{0.33\textwidth}
        \includegraphics[width=\linewidth]{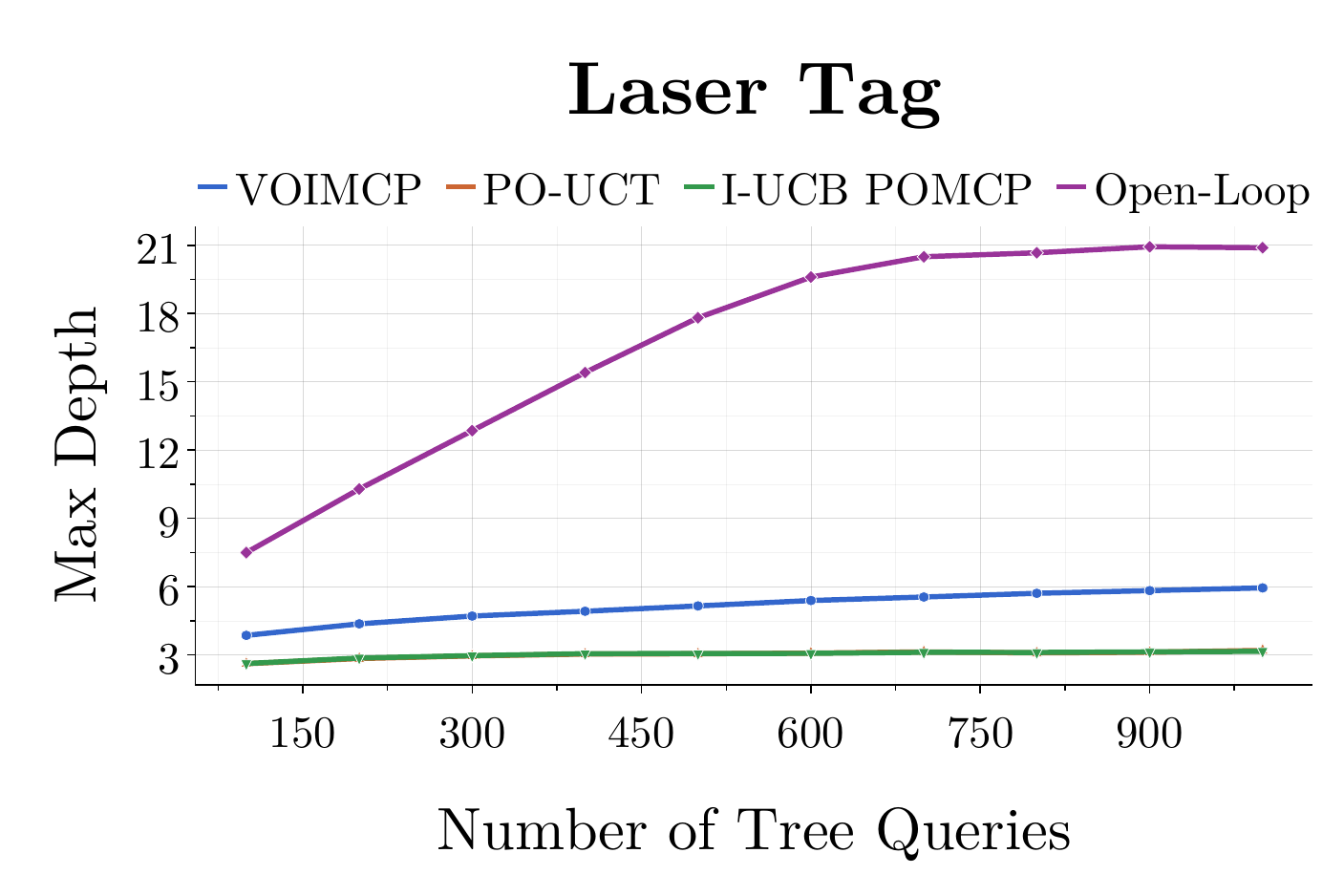}
        \end{subfigure}

  \begin{subfigure}[b]{0.33\textwidth}
        \includegraphics[width=\linewidth]{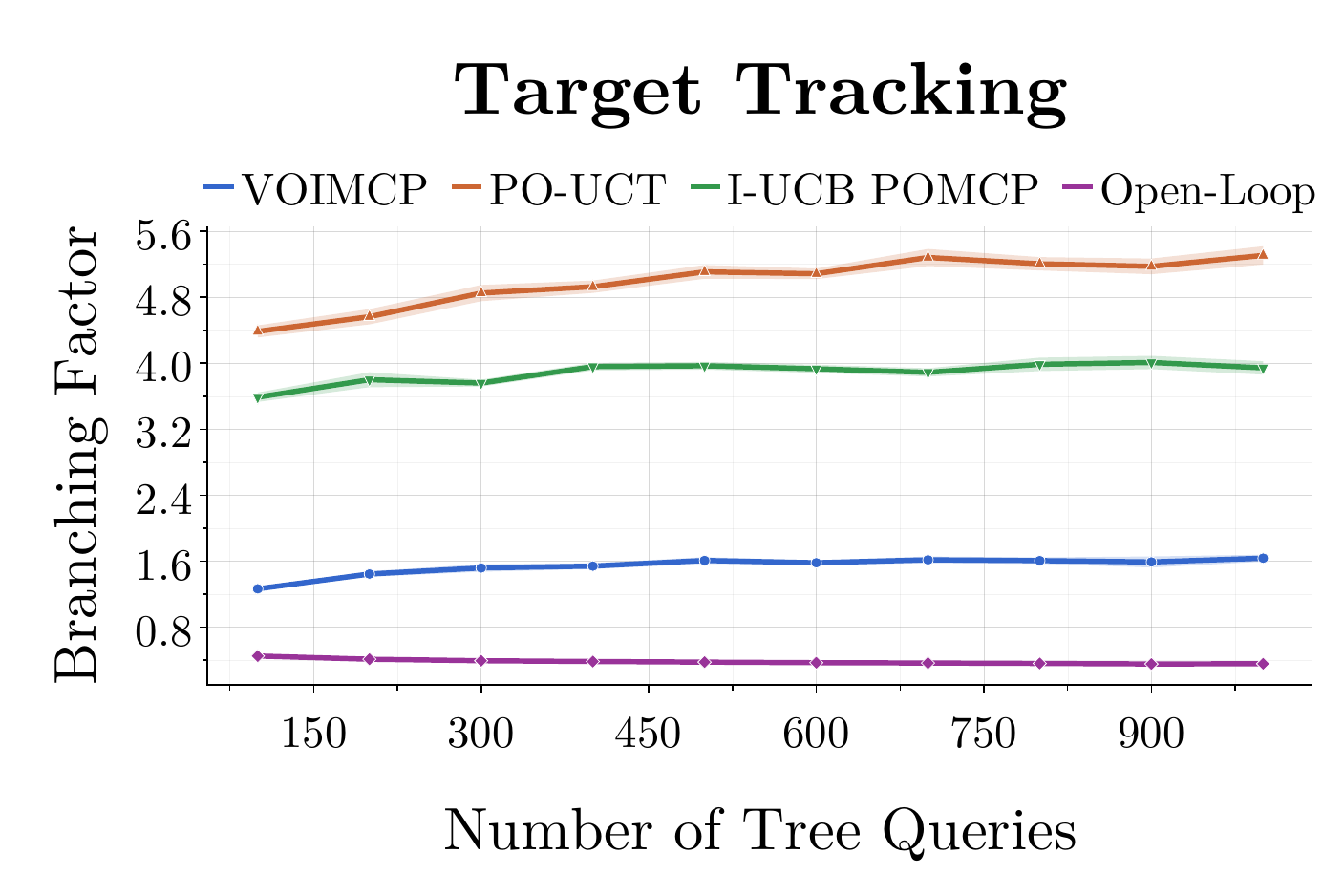}
    \end{subfigure}
    \begin{subfigure}[b]{0.33\textwidth}
        \includegraphics[width=\linewidth]{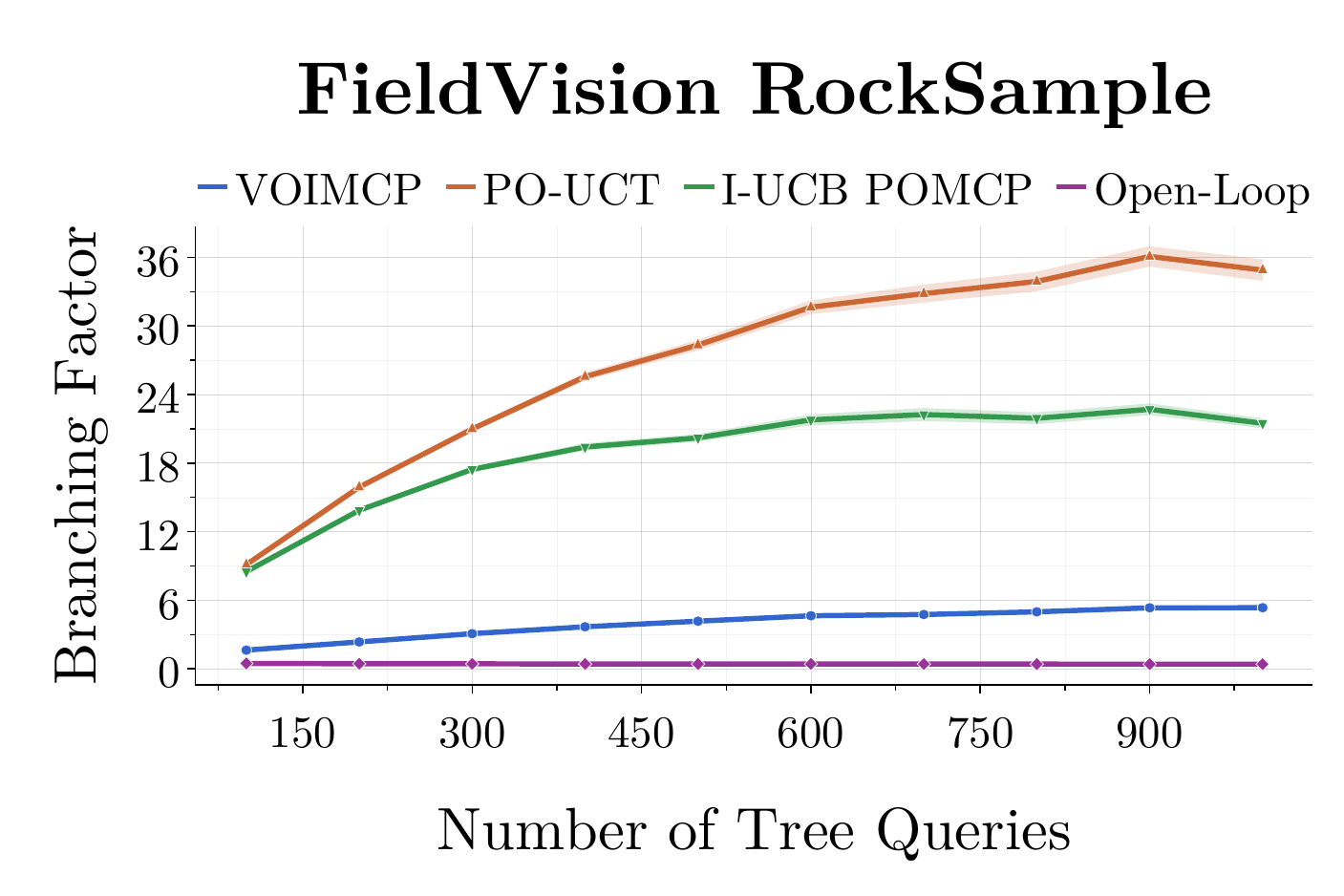}
        \end{subfigure}
    \begin{subfigure}[b]{0.33\textwidth}
        \includegraphics[width=\linewidth]{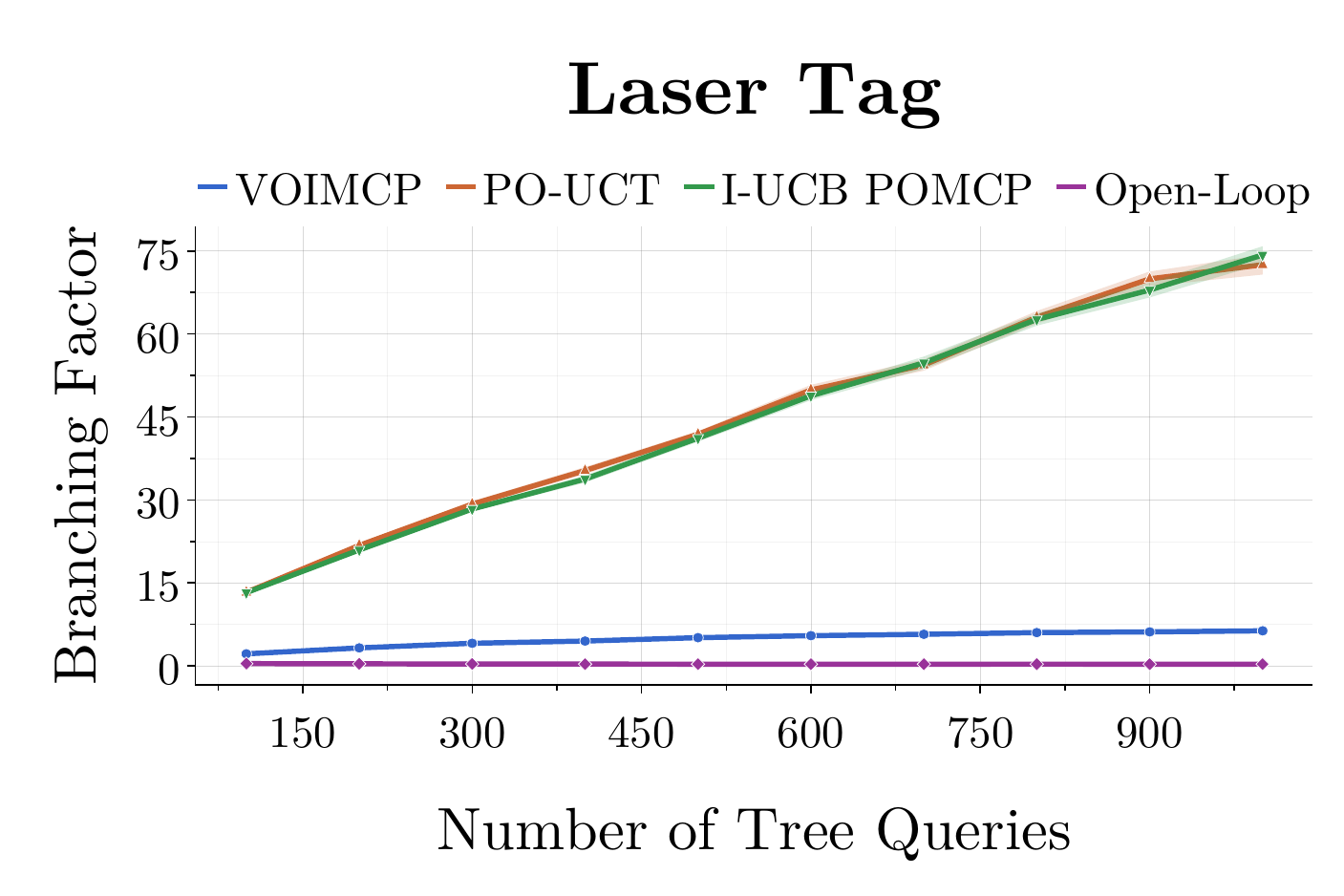}
        \end{subfigure}
    \caption{Tree growth statistics. (Top Row) Maximum tree depth vs. number of tree queries. (Bottom Row) Effective action-observation branching factor vs. number of tree queries. Statistics are computed over 100 trials.}
    \label{fig:statistics}
\end{figure*}

\subsection{Results and Discussion}

In Figure~\ref{fig:benchmark results}, we report the discounted cumulative reward against the number of tree queries. We also report the mean maximum tree depth and effective branching factor in Figure~\ref{fig:statistics}. We define the maximum tree depth as the greatest depth reached by any history node in the search tree, and the effective branching factor as the average number of action-observation branches expanded per visited node.

VOIMCP outperforms all baselines in mean discounted cumulative reward across all the benchmark problems. The statistics in Figure~\ref{fig:statistics} reveal the mechanism behind this success: with the exception of the Open-Loop policy which only has $1$ observation branch per action, VOIMCP consistently achieves the highest maximum tree depth while maintaining the lowest effective branching factor. These results confirm that the adaptive VOI framework enables VOIMCP to ``take shortcuts'' via selectively avoiding expanding dense observation branches (closed-loop actions) whenever the value of information is low. By selectively reducing the computational overhead of observation processing, VOIMCP allocates limited planning resources to enable deeper search. This alleviates the curse of history and facilitates a more efficient exploration, leading to superior performance.

In domains with large observation spaces, the performance of PO-UCT degrades. This occurs because the original UCB heuristic distributes exploration across the observations, leading to shallower search trees than VOIMCP. Interestingly, the Open-Loop policy performs better than I-UCB POMCP in both the Target Tracking and FieldVision RockSample problems, and better than PO-UCT in the latter. This counter-intuitive result provides an insight: when observation spaces are large and noisy, the computational cost of fully processing observations may outweigh the benefit. By always ignoring observations, the Open-Loop policy eliminates observation branching and searches deeply, but cannot reason about observation information. VOIMCP bridges this gap through the adaptive VOI, acting open-loop when observations are noisy or irrelevant, while retaining the reactivity of closed-loop planning when necessary.

While the I-POMCP algorithm has demonstrated superior performance on some domains \cite{alves2023ibpomcp}, its planning component I-UCB POMCP was the worst-performing baseline across all three benchmarks. I-UCB augments the traditional UCB with an observation entropy heuristic. This term is designed to induce information-gathering behavior by rewarding actions that heuristically reduce uncertainty. However, our results demonstrate that the observation entropy is an insufficient proxy for the value of information. I-UCB may prioritize high-entropy observations that are irrelevant to the task rewards. In contrast, VOIMCP directly searches based on the value of information, ensuring that the agent only seeks observations that contribute to reward maximization.

The effectiveness of VOIMCP relies on the search-depth gains from avoiding observation branching outweighing the increase of the augmented action space. For many POMDPs, such as our case studies, the existence of open-loop shortcuts enables more efficient search. However, for problems in which the value of information is high but poorly captured during search, the selection policy may persistently select open-loop actions that are counterproductive for high-value policies. We posit that VOI-guided pruning is most effective in domains where the computational curse of history is the primary bottleneck to discovering high-value policies.
\section{Conclusion}

This paper presents a recursive framework to reason about the VOI. We propose VOIMCP, an algorithm that selectively disregards observation information when the VOI is low. We prove the theoretical properties of both the framework and our algorithm, showing bounded regret and non-asymptotic convergence. This work shows that value of information reasoning can reduce the effective branching factor, alleviating the curse of history in POMDPs. Future work includes investigating fundamental structures of POMDPs that are most amenable to VOI reasoning, designing better schemes for tuning and annealing $\kappa$, and integrating the adaptive VOI reasoning into other POMDP solution techniques.

\section{Acknowledgments}

This work was partially supported by the National Science Foundation, Grants 2340958 and 2137269 and the members of the Center for Autonomous Air Mobility and Sensing (CAAMS) IUCRC.

\bibliography{aaai2026}
\section{Proof of Theorem \ref{thm:VOIMCP convergence}}

Our analysis relies on a sequence of theoretical results analogous to Lemmas $1-7$ and Theorems $1$ and $3$ in \citet{shah2022polyuct}. As the algebraic structure remains largely invariant, we omit repetitive derivations. The key deviations arise from our definition of an optimal conditional VOI action and the selection strategy, both of which depend on $\kappa$, propagating changes to the constants in the final results. We provide the theoretical statements with updated variables below, but detail the full proofs only where the deviation is non-trivial.

We first start with the analysis of a class of nonstationary multi-armed bandit as defined in \citet{shah2022polyuct}. Let there be $K \geq 1$ arms or actions. Let $\bar{X}_{i,n} = \frac{1}{n}\sum_{t=1}^{n}X_{i,t}$ denote the empirical average of choosing arm $i$ for $n$ times, and let $\mu_{i,n}$ be its expectation. For each arm $i \in [K]$, the reward $X_{i,t}$ is bounded in $[-R, R]$ for some $R > 0$, with the reward sequence being a nonstationary process satisfying:

\noindent \textbf{A. Convergence}: The expectation $\mu_{i,n}$ converges to a value $\mu_i$, that is,
\begin{align}
    \mu_i = \lim_{n\rightarrow\infty}\mathbb{E}[\bar{X}_{i,n}]\label{eq: MAB convergence}
\end{align}
\textbf{B. Concentration}: There exist constants $\beta>1$, $\xi>0$, and $1/2 \le \eta < 1$ such that for any $n\ge1$ and $z\ge1$:
\begin{align}
\mathbb{P}\big(n(\bar X_{i,n}-\mu_i) \ge n^\eta z\big) \le \beta z^{-\xi},\\
\mathbb{P}\big(n(\bar X_{i,n}-\mu_i) \le - n^\eta z\big) \le \beta z^{-\xi} \label{eq: MAB concentration}.
\end{align}

Distinct from \citet{shah2022polyuct}, we keep track of the modified arm mean biased by a negative fraction $\kappa$ of its absolute value: 
\begin{align}
\mu_i^U \equiv \mu_i - \kappa_i \lvert \mu_i \rvert,
\label{eq:voi_mean}
\end{align}
We define the new conditional VOI optimal value with respect to the $\kappa$-biased converged expectation
\begin{align}
    \mu^U_* \equiv \max_{i \in [K]}\mu_i^U,
    \label{eq:optimal_voi_mean}
\end{align}
where $\kappa_i \in [0,1]$ may be different for each arm. This will be useful for our VOI-POMDP since it allows open-loop actions to be optimal if their expected reward is within a $\kappa\mhyphen$fraction of their closed-loop counterparts. We define $i^* \in \arg\max_{i \in [K]}\mu_i^U$ and assume the optimal arm is unique. Let the conditional VOI-based suboptimality gap for a suboptimal arm $i$ be:
\begin{align}
    \Delta_i^U \equiv \mu^U_* - \mu_i^U > 0.
\end{align}
and the minimum gap $\Delta_{min}^U=min_{i\in[K],i\neq i^*}\Delta_i^U$. 

Consider the VOI-UCB:
\begin{align}
\label{eq: generalized UCB}
U_i(s_i,t) \equiv \bar X_{i,s_i} - \kappa_i \lvert \bar X_{i,s_i} \rvert+ B_{t,s_i}.
\end{align}
The bonus or \textit{bias} $B_{t,s_i}$ is specified as:
\begin{align}
    B_{t,s_i} = \frac{\beta^{1/\xi} \cdot t^{\alpha/\xi}}{s_i^{1-\eta}},
\end{align}
where $t$ is the total number of pulls and $s_i$ is the number of pulls of arm $i$. Then we define the VOI-UCB algorithm as:

\begin{align}
    I_t \in \arg\max_{i \in [K]} U_i(s_i,t),
    \label{eq:action_selection_policy}
\end{align}
where a tie between arms is broken arbitrarily. Let $\delta_{i^*,n}=\mu_{i^*,n} - \mu_{i^*}$, which measures how fast the mean of the VOI-optimal arm converges. For simplicity, we sometimes drop the arm index $i$ from quantities related to the optimal arm $i^*$, e.g. $\delta_{i^*,n}=\delta_{*,n}$.

With the following useful Lemma, we have that the probability that a suboptimal arm or action has a large upper confidence is polynomially small.

\begin{lemma}
\label{lemma: VOI-UCB bounded probability}
Let $i \in [K]$ be a suboptimal arm. Define
\begin{align}
    \label{eq:newAi}
    A_i(t) \equiv & \min_{u \in \mathbb{N}}\Big\{B_{t,u} \le \frac{\Delta_i^U}{2(1+\kappa_i)} \Big\} \nonumber \\ = & 
    \left\lceil \left(\frac{2(1+\kappa_i)}{\Delta^U_i} \cdot \beta^{1/\xi} \cdot t^{\alpha/\xi}\right)^{\frac{1}{1-\eta}} \right\rceil
\end{align}
For any $t \ge 1$, $t \ge s_i \ge A_i(t)$, we have
\begin{align}
    \mathbb{P}\big(U_i(s_i,t) > \mu^U_*\big) \le t^{-\alpha}.
\end{align}
\end{lemma}
\begin{proof}
Fix $t \ge 1$ and $s_i \ge A_i(t)$. Define $f(x) = x - \kappa_i |x|$ and note that $\mu_i^U = f(\mu_i)$ and the exploitation component of $U_i$ satisfies $\bar{X}_{i,s_i} - \kappa_i|\bar{X}_{i,s_i}| = f(\bar{X}_{i,s_i})$. We can decompose
\begin{align*}
    U_i(s_i,t) - \mu^U_* 
    &= \big[f(\bar{X}_{i,s_i}) - f(\mu_i)\big] + B_{t,s_i} - \Delta_i^U.
\end{align*}
Since $s_i \ge A_i(t)$, by definition $B_{t,s_i} \le \Delta_i^U / 2(1+\kappa_i) \le \Delta_i^U/2$, so
\begin{align*}
    B_{t,s_i} - \Delta_i^U \le -\frac{\Delta_i^U}{2}.
\end{align*}
Therefore, $U_i(s_i,t) > \mu^U_*$ requires $f(\bar{X}_{i,s_i}) - f(\mu_i) > \Delta_i^U / 2 > 0$. Since $f$ is non-decreasing, the strict inequality 
$f(\bar{X}_{i,s_i}) - f(\mu_i) > \Delta_i^U/2 > 0$ implies 
$\bar{X}_{i,s_i} > \mu_i$. Furthermore, for any $y > x$, we have the Lipschitz bound
\begin{align*}
    f(y) - f(x) \le (1+\kappa_i)(y - x),
\end{align*}
since the steepest slope of $f$ is $1+\kappa_i$. Applying this with $y = \bar{X}_{i,s_i}$ and $x = \mu_i$ yields
\begin{align*}
    f(\bar{X}_{i,s_i}) - f(\mu_i) \le (1+\kappa_i)(\bar{X}_{i,s_i} - \mu_i).
\end{align*}
Thus, whenever $U_i(s_i,t) > \mu^U_*$ we must have
\begin{align*}
    (1+\kappa_i)(\bar{X}_{i,s_i} - \mu_i) > \frac{\Delta_i^U}{2},
\end{align*}
i.e.,
\begin{align*}
    \bar{X}_{i,s_i} - \mu_i > \frac{\Delta_i^U}{2(1+\kappa_i)} \ge B_{t,s_i},
\end{align*}
where the last inequality follows from $s_i \ge A_i(t)$. Therefore,
\begin{align*}
    \mathbb{P}\big(U_i(s_i,t) > \mu^U_*\big) 
    &\le \mathbb{P}\Big(\bar{X}_{i,s_i} - \mu_i > B_{t,s_i}\Big) 
    \le t^{-\alpha},
\end{align*}
using the concentration assumption~\eqref{eq: MAB concentration}.
\end{proof}

\begin{lemma}
\label{lemma: VOI-UCB bounded probability lower tail}
For any $t \ge 1$ and $1 \le s_* \le t$,
\[
\mathbb{P}\!\left(U_{*}(s_*,t) \le \mu_*^U\right)
\le (1+\kappa_{*})^{\xi} t^{-\alpha},
\]
provided $\beta$ is chosen large enough so that
$\beta^{1/\xi} \ge 1+\kappa_{*}$.
\end{lemma}

\begin{proof}
Define
\[
f_{*}(x) \equiv x - \kappa_{*}|x|.
\]
Then $f_{*}$ is nondecreasing on $\mathbb{R}$ and
\[
U_{*}(s_*,t) = f_{*}(\bar X_{*,s_*}) + B_{t,s_*},
\qquad
\mu_*^U = f_{*}(\mu_{*}).
\]
Hence
\[
\mathbb{P}\!\left(U_{*}(s_*,t) \le \mu_*^U\right)
=
\mathbb{P}\!\left(f_{*}(\bar X_{*,s_*}) + B_{t,s_*}
\le f_{*}(\mu_{*})\right).
\]
On this event, since $B_{t,s_*}>0$, we must have
$f_{*}(\bar X_{*,s_*}) < f_{*}(\mu_{*})$, and therefore,
by monotonicity of $f_{*}$,
\[
\bar X_{*,s_*} \le \mu_{*}.
\]
Now, for any $x \le y$,
\[
f_{*}(y) - f_{*}(x)
\le (1+\kappa_{*})(y-x),
\]
since the maximum slope of $f_{*}$ is $1+\kappa_{*}$.
Applying this with $x=\bar X_{*,s_*}$ and $y=\mu_{*}$ gives
\[
f_{*}(\mu_{*}) - f_{*}(\bar X_{*,s_*})
\le (1+\kappa_{*})(\mu_{*}-\bar X_{*,s_*}).
\]
Therefore, whenever
$f_{*}(\bar X_{*,s_*}) + B_{t,s_*} \le f_{*}(\mu_{*})$,
we have
\[
B_{t,s_*}
\le f_{*}(\mu_{*}) - f_{*}(\bar X_{*,s_*})
\le (1+\kappa_{*})(\mu_{*}-\bar X_{*,s_*}),
\]
which implies
\[
\bar X_{*,s_*} - \mu_{*}
\le -\frac{B_{t,s_*}}{1+\kappa_{*}}.
\]
Thus,
\[
\mathbb{P}\!\left(U_{*}(s_*,t) \le \mu_*^U\right)
\le
\mathbb{P}\!\left(
\bar X_{*,s_*} - \mu_{*}
\le -\frac{B_{t,s_*}}{1+\kappa_{*}}
\right).
\]
Multiplying by $s_*$ and matching the concentration form in \eqref{eq: MAB concentration},
set
\[
z
=
s_*^{\,1-\eta}\frac{B_{t,s_*}}{1+\kappa_{*}}
=
s_*^{\,1-\eta}
\frac{\beta^{1/\xi} t^{\alpha/\xi}}{s_*^{\,1-\eta}(1+\kappa_{*})}
=
\frac{\beta^{1/\xi}}{1+\kappa_{*}} t^{\alpha/\xi}.
\]
If $\beta^{1/\xi} \ge 1+\kappa_{*}$, then $z \ge 1$ for all $t \ge 1$.
Applying the lower-tail concentration assumption \eqref{eq: MAB concentration},
\begin{align*}
    \mathbb{P}\!\left(U_{*}(s_*,t) \le \mu_*^U\right)
    \le \beta z^{-\xi}
    = &
    \beta
    \left(
    \frac{\beta^{1/\xi}}{1+\kappa_{*}} t^{\alpha/\xi}
    \right)^{-\xi} \\
    = &
    (1+\kappa_{*})^\xi t^{-\alpha}.
\end{align*}

\end{proof}

Lemma~\ref{lemma: VOI-UCB bounded probability} and \ref{lemma: VOI-UCB bounded probability lower tail} allow us to upper bound the expected number of suboptimal arm pulls as follows.

\begin{lemma}[Adapted from Lemma $2$ of \cite{shah2022polyuct}]
\label{lemma: VOI-UCB bounded expected pulls}
Let $i \in [K]$, $i \neq i^*$ be a suboptimal arm. Then, for all $n \ge 1$, the expected number of pulls of arm $i$ satisfies
\begin{align*}
\mathbb{E}[T_i(n)]
    &\le \Big( \frac{2(1+\kappa_i)}{\Delta_i^U} \beta^{1/\xi} \Big)^{1/(1-\eta)} n^{\alpha/(\xi(1-\eta))} \\ & \quad +  \frac{1+(1+\kappa_{i^*})^\xi}{\alpha-2} + 1.
\end{align*}
\end{lemma}
\begin{proof}
    The proof follows \citet{shah2022polyuct}(cf. Lemma 2) after replacing Shah’s threshold
    $A_i(t)$ by our threshold in \eqref{eq:newAi}, replacing Shah’s suboptimal-arm upper-tail
    bound by Lemma~\ref{lemma: VOI-UCB bounded probability}, and replacing Shah’s optimal-arm lower-tail
    bound by Lemma~\ref{lemma: VOI-UCB bounded probability lower tail}. With these substitutions, the
    decomposition at $T_i(t)\ge A_i(n)$, the summation over polynomial tails, and the conversion
    to the expected-pulls bound proceed identically up to constants.
\end{proof}

Let $\bar X_n \equiv \frac{1}{n} \sum_{i=1}^K T_i(n)\bar X_{i,T_i(n)}$ denote the empirical average under the VOI-UCB algorithm \eqref{eq:action_selection_policy}. Then, $\bar X_n$ satisfies the following convergence and concentration properties.
\begin{theorem}[Adapted from Theorem $3$ of \citep{shah2022polyuct}]
\label{theorem: MAB VOI-UCB Convergence}
Consider a nonstationary MAB satisfying~\eqref{eq: MAB convergence} and~\eqref{eq: MAB concentration}. 
Suppose that Algorithm~\eqref{eq:action_selection_policy} is applied with parameter 
$\alpha$ such that $\xi \eta (1-\eta) \le \alpha < \xi(1-\eta)$ and $\alpha > 2$. 
Then, the following holds:

\noindent \textbf{A. Convergence}:
\begin{align*}
\big|\mathbb{E}[\bar X_n] & - \mu_{*} \big|
\le |\delta_{*,n}| \; + \\
& \hspace{-4mm}\frac{2R(K-1)}{n} 
\Big(
\big(\frac{4}{\Delta^U_{min}} \beta^{\frac{1}{\xi}} \big)^{\frac{1}{1-\eta}} \cdot n^{\frac{\alpha}{(\xi(1-\eta))}} + \frac{1+2^\xi}{\alpha-2} + 1\Big),
\end{align*}

\noindent \textbf{B. Concentration}: there exists constants, $\beta' > 1$ and $\xi' > 0$ and $1/2 \leq \eta' < 1$ such that for every $n \geq 1$ and every $z \geq 1$,
\begin{align*}
    \mathbb{P}(n\bar{X}_n - n\mu_{*} \geq n^{\eta'}z) \leq \frac{\beta'}{z^{\xi'}},\\
    \mathbb{P}(n\bar{X}_n - n\mu_{*} \leq -n^{\eta'}z) \leq \frac{\beta'}{z^{\xi'}},
\end{align*}
where $\eta' = \frac{\alpha}{\xi(1-\eta)}, \xi'=\alpha - 1, \beta'$ depends on $R, K, \Delta^U_{min}, \beta, \xi, \alpha, \eta, \kappa$. 
\end{theorem}

\begin{proof}
The proof follows the proof of Theorem~3 in \citet{shah2022polyuct} after replacing Shah’s threshold by
\[
A(t)=\max_{i\in[K]}A_i(t)
=
\left\lceil
\left(\frac{4}{\Delta^U_{\min}}\beta^{1/\xi}\right)^{\frac{1}{1-\eta}}
t^{\frac{\alpha}{\xi(1-\eta)}}
\right\rceil,
\]
replacing Shah’s suboptimal-arm upper-tail bound by Lemma~\ref{lemma: VOI-UCB bounded probability}, replacing Shah’s optimal-arm lower-tail bound by Lemma~\ref{lemma: VOI-UCB bounded probability lower tail}, and using Lemma~\ref{lemma: VOI-UCB bounded expected pulls} in place of Shah’s expected-pulls bound. With these substitutions, the remainder of the convergence and concentration argument is identical up to constants, yielding the stated result.
\end{proof}

We aim to prove that under our VOI-based MCTS (adapted from Algorithm $1$ in \citet{shah2022polyuct}), the mean of the empirical reward at the root node of the MCTS tree is within $O(n^{\eta-1})$ of the mean reward obtained as the number of samples tends to infinity. To construct the VOI-based MCTS, we replace the standard action selection policy (cf. (5) of \citet{shah2022polyuct}) with the following $\kappa$-biased policy:

\begin{equation}
    \begin{aligned}
        a^{(h+1)} = & \arg \max_{a\in \mathcal{A}} Q^h_N(s,a) - \kappa_a|Q^h_N(s,a)| + B^h_N(s,a)
    \end{aligned}
\end{equation}
where
\begin{align}
    Q^h_N(s,a) = \frac{q^{(h+1)}(s^{(h)},a) + \gamma \tilde v^{(h+1)}(s^{(h)}\circ a)}{N^{(h+1)}(s^{(h)}\circ a)}
\end{align}
and
\begin{align}
    B^h_N(s,a) = \frac{(\beta^{(h+1)})^{1/\xi^{(h+1)}} \cdot (N^{(h)}(s^{(h)}))^{\alpha^{(h+1)}/\xi^{(h+1)}}}{(N^{(h+1)}(s^{(h)}\circ a))^{1-\eta^{(h+1)}}}
\end{align}
Note, $\kappa_a$ is parameterized by action $a$ where for each $a\in[K]$, $\kappa_a\in[0,1]$

\subsubsection{Analyzing Leaf Level $d = D$}
Let $n_{D-1}$ denote the number of nodes at depth $D-1$, corresponding to states $s_{1,D-1},\dots,s_{n_{D-1},D-1} \in \mathcal{S}$. Consider node $i\in[n_{D-1}]$ at level $D-1$ corresponding to state $s_{i,D-1}$. Under the VOI-MCTS algorithm, whenever this node is visited, an action $a\in[K]$ will be taken, reaching the leaf node ${s}_{D}' = s_{i,D-1} \circ a$. This results in the cumulative reward $\mathcal{R}(s_{i,D-1},a) + \gamma \tilde v^{(D)}({s}_{D}')$, where $\tilde{v}^{(D)}(\cdot)$ denotes an initial value estimate $V_{\epsilon_0} = 0$. Similar to \eqref{eq:voi_mean} and \eqref{eq:optimal_voi_mean} for the bandit, we define the following VOI-based terms for the tree:

\begin{equation}
    \begin{aligned}
        \mu_a^{(D\mhyphen1)}(s_{i,D\mhyphen1}) = & \mathbb{E}\left[ \mathcal{R}(s_{i,D\mhyphen1}, a)\right] + \gamma \tilde v^{(D)}({s}_{i,D\mhyphen1} \circ a), \\
        \mu_a^{U,(D\mhyphen1)}(s_{i,D\mhyphen1}) = &\mu_a^{(D\mhyphen1)}(s_{i,D\mhyphen1}) - \kappa_a^{(D\mhyphen1)}|\mu_a^{(D\mhyphen1)}(s_{i,D\mhyphen1})|,\\
        a_*^{(D\mhyphen1)}(s_{i,D\mhyphen1}) \in & \arg \max_{a\in [K]} \mu_a^{U,(D\mhyphen1)}(s_{i,D\mhyphen1}), \\
        \mu_*^{U,(D\mhyphen1)}(s_{i,D\mhyphen1}) = & \max_{a\in[K]} \mu_a^{U,(D\mhyphen1)}(s_{i,D\mhyphen1}),\\
        \mu_*^{(D\mhyphen1)}(s_{i,D\mhyphen1}) = & \mu_{a_*^{(D\mhyphen1)}(s_{i,D\mhyphen1})}^{(D\mhyphen1)}(s_{i,D\mhyphen1}),\\
        \Delta_{min}^{U,(D \mhyphen 1)}(s_{i,D\mhyphen1}) = & \mu_*^{U,(D\mhyphen1)}(s_{i,D\mhyphen1}) \\
        & - \max_{a\neq a_*^{(D\mhyphen1)}(s_{i,D\mhyphen1})} \mu_a^{U,(D\mhyphen1)}(s_{i,D\mhyphen1}), \\
    \end{aligned}
    \label{eq:leaf_node_stats}
\end{equation}
where all rewards belong to $[-\tilde{R}_{\max}^{(D-1)}, \tilde{R}_{\max}^{(D-1)}]$ consistent with \citet{shah2022polyuct}.

By applying Lemma 4 from \citet{shah2022polyuct}, we obtain that for $\eta^{(D)} \in \left[\tfrac{1}{2}, 1\right)$, and sufficiently large $\xi^{(D)}$, 
there exists a constant $\beta^{(D)}$ such that the collected rewards at 
$s_{i,D-1}$ (i.e., the sum of the i.i.d. stochastic reward and the value 
evaluation) satisfy the convergence property (cf.~\eqref{eq: MAB convergence}) and the concentration 
property (cf.~\eqref{eq: MAB concentration}). Therefore, by an application of Theorem~\ref{theorem: MAB VOI-UCB Convergence}, we have the following Lemma.

\begin{lemma}[Adapted from Lemma $5$ of \citet{shah2022polyuct}]
    \label{lemma: leaf convergence}
    Consider a node corresponding to state $s_{i,D-1}$ at level $d = D-1$ within the VOI-MCTS for $i \in [n_{D-1}]$. Let $\tilde{v}^{(D-1)}(s_{i,D-1})_n$ be the total discounted reward collected at $s_{i,D-1}$ during $n \ge 1$ visits of it, to one of its $K$ leaf nodes under the VOI-UCB policy. Then, for the choice of appropriately large $\beta^{(D)} > 0$, for a given $\xi^{(D)} > 0$, $\eta^{(D)} \in [\frac{1}{2}, 1)$ and $\alpha^{(D)} > 2$, we have

\noindent \textbf{A. Convergence:}
{\small
\begin{align*}
    &\left|\mathbb{E}\left[\frac{1}{n} \tilde{v}^{\!(D-1)}(s_{i,D-1})_n\right] - \mu^{(D-1)}_*(s_{i,D-1})\right| \\
    &\le \psi^{(D-1)}_{n}\!\!\!\left( \left(\frac{4(\beta^{(D))^\frac{1}{\xi^{(D)}}}}{\Delta_{min}^{U,(D-1)}(s_{i,D-1})}\right)^{\frac{1}{1 - \eta^{(D)}}} \!\!\!\!\!\!\cdot n^{\frac{\alpha^{(D)}}{\xi^{(D)}(1-\eta^{(D)})}} + C^{(D)}\right),
\end{align*}
}
where $\psi^{(D-1)}_{n} = \frac{2\tilde{R}_{\max}^{(D-1)} (K-1)}{n} $ and $C^{(D)} = \frac{1+2^{\xi^{(D)}}}{\alpha^{(D)}-2} + 1 $.

\noindent \textbf{B. Concentration:} there exist constants, $\beta' > 1$ and $\xi' > 1$ and $1/2 \le \eta' < 1$ such that for every $n \ge 1$ and every $z \ge 1$,
\begin{align*}
    \mathbb{P}(\tilde{v}^{(D-1)}(s_{i,D-1})_n - n\mu_{*}^{(D-1)}(s_{i,D-1}) &\ge n^{\eta'} z) \le \frac{\beta'}{z^{\xi'}}, \\
    \mathbb{P}(\tilde{v}^{(D-1)}(s_{i,D-1})_n - n\mu_{*}^{(D-1)}(s_{i,D-1}) &\le -n^{\eta'} z) \le \frac{\beta'}{z^{\xi'}},
\end{align*}
where $\eta' = \frac{\alpha^{(D)}}{\xi^{(D)}(1-\eta^{(D)})}$, $\xi' = \alpha^{(D)} - 1$, and $\beta'$ is a large enough constant that is function of parameters $\alpha^{(D)}, \beta^{(D)}, \xi^{(D)}, \eta^{(D)}, \tilde{R}_{\max}^{(D-1)}, K$, \text{and} $\Delta_{\min}^{U,(D-1)}(s_{i,D\mhyphen1})$.
\end{lemma}

\subsubsection{Recursion: Going from level $d$ to $d-1$}

The proof for the convergence and concentration properties of the value estimates across all levels $d$ of the VOI-MCTS tree proceeds via mathematical induction, following the structure established by \citet{shah2022polyuct}. The mechanism of the induction is identical to the original paper, but the final, global convergence rate differs due to how we defined the sample threshold $A_{i}(t)$.

Assume the estimated value $\tilde{v}^{(d)}(\cdot)$ at level $d$ satisfies properties \eqref{eq: MAB convergence} and \eqref{eq: MAB concentration} with parameters $\alpha^{(d)}$, $\xi^{(d)}$, $\eta^{(d)}$, and an appropriately large constant $\beta^{(d)}$. For any node $s_{i,d-1}$ at level $d-1$, the effective reward received is the combined value $\mathcal{R}(s_{i,d-1}, a) + \gamma \tilde{v}^{(d)}(s'_d)$ where ${s}_{d}' = s_{i,d-1} \circ a$, and $\tilde{v}^{(d)}(\cdot)$ returns the collected rewards recursively until the leaf node at depth $D$. By the inductive hypothesis and an application of Lemma $4$ \citep{shah2022polyuct}, this combined reward sequence continues to satisfy assumptions \eqref{eq: MAB convergence} and \eqref{eq: MAB concentration} with the same $\alpha, \xi, \eta$ dependencies and a re-adjusted constant $\beta^{(d-1)}$.

Note, $\mu_*^{(d-1)}(s_{i,d-1})$, $\mu^{U,(d-1)}_{a}(s_{i,d-1}), \mu^{U,(d-1)}_{*}(s_{i,d-1})$, $a_*^{(d-1)}(s_{i,d-1})$, and $\Delta_{\min}^{U,(d-1)}(s_{i,d-1})$ are defined similarly to \eqref{eq:leaf_node_stats}, and rewards here belong to $[-\tilde{R}_{\max}^{(d-1)}, \tilde{R}_{\max}^{(d-1)}]$. By an application of Theorem~\ref{theorem: MAB VOI-UCB Convergence}, we obtain the following result regarding the desired convergence and concentration properties of the combined reward sequence at level $d-1$, which completes the induction and proves the results for all $d \geq 0$. 

\begin{lemma}[Adapted from Lemma $6$ of \citet{shah2022polyuct}]
    \label{lemma: recursive convergence}
    Consider a node corresponding to state $s_{i,d-1}$ at level $d-1$ within the VOIMCP for $i \in [n_{d-1}]$. Let $\tilde{v}^{(d-1)}(s_{i,d-1})_n$ be the total discounted reward collected at $s_{i,d-1}$ during $n \ge 1$ visits. Then, for the choice of appropriately large $\beta^{(d)} > 0$, for a given $\xi^{(d)} > 0$, $\eta^{(d)} \in [\frac{1}{2}, 1)$ and $\alpha^{(d)} > 2$, we have

\noindent \textbf{A. Convergence:}
{\small
\begin{align*}
    &\left|\mathbb{E}\left[\frac{1}{n} \tilde{v}^{(d-1)}(s_{i,d-1})_n\right] - \mu^{(d-1)}_{*}(s_{i,d-1})\right| \\
    &\le \psi^{(d-1)}_{n}\left( \left(\frac{4(\beta^{(d)})^{\frac{1}{\xi^{(d)}}}}{\Delta_{\min}^{U,(d-1)}(s_{i,d-1})}\right)^{\frac{1}{1-\eta^{(d)}}} \cdot n^{\frac{\alpha^{(d)}}{\xi^{(d)}(1-\eta^{(d)})}} + C^{(d)} \right),
\end{align*}
}
where $\psi^{(d-1)}_{n} = \frac{2 \tilde{R}_{\max}^{(d-1)} (K-1)}{n}$ and $C^{(d)} = \frac{1+2^{\xi^{(d)}}}{\alpha^{(d)} - 2} + 1$.

\noindent \textbf{B. Concentration:} there exist constants, $\beta' > 1$ and $\xi' > 1$ and $1/2 \le \eta' < 1$ such that for $n \ge 1$, $z \ge 1$,
\begin{align*}
    \mathbb{P}(\tilde{v}^{(d-1)}(s_{i,d-1})_n - n\mu^{(d-1)}_{*}(s_{i,d-1}) &\ge n^{\eta'} z) \le \frac{\beta'}{z^{\xi'}}, \\
    \mathbb{P}(\tilde{v}^{(d-1)}(s_{i,d-1})_n - n\mu^{(d-1)}_{*}(s_{i,d-1}) &\le -n^{\eta'} z) \le \frac{\beta'}{z^{\xi'}},
\end{align*}
where $\eta' = \frac{\alpha^{(d)}}{\xi^{(d)}(1-\eta^{(d)})}$, $\xi' = \alpha^{(d)} - 1$, and $\beta'$ is a large enough constant that is function of parameters $\alpha^{(d)}, \beta^{(d)}, \xi^{(d)}, \eta^{(d)}, \tilde{R}_{\max}^{(d-1)}, K, \Delta_{\min}^{U,(d-1)}(s_{i,d-1})$.
\end{lemma}

\subsubsection{Completing Proof of Theorem~\ref{thm:VOIMCP convergence}}

Using Lemma~\ref{lemma: recursive convergence}, we conclude that the recursive
relationship going from $d$ to $d-1$ holds for all $d\geq 1$, with level $0$
being the root. Thus, for the root state $s^{(0)}$, under $n$ total simulations
of VOI-MCTS, the empirical average of the rewards $\frac{1}{n}\tilde v^{(0)}(s_0)_n$
satisfies
\begin{align}
    \left|
    \mathbb{E}\!\left[\frac{1}{n}\tilde v^{(0)}(s_0)_n\right]
    - \mu_*^{(0)}(s_0)
    \right|
    &\le
    O\!\left(
    n^{\frac{\alpha^{(0)}}{\xi^{(0)}(1-\eta^{(0)})}-1}
    \right)
    \nonumber\\
    &= O(n^{\eta-1}),
\end{align}
where $\mu_*^{(0)}(s_0)$ is the raw depth-$D$ value obtained at the root by the
$\kappa$-adaptive VOI-MCTS recursion.

To transfer this root-level VOI-MCTS result to the VOIMCP estimator
$\bar{V}_n(b_0)=\frac{1}{n}\sum_{a\in\mathcal{A}'}N(h_0a)\bar Q(h_0a)$, we introduce the derived history-MDP
for the transformed VOI-POMDP $\mathcal P'$. Let
\[
\tilde{\mathcal M}' = (\mathcal H', \mathcal A', \bar{\mathcal T}, \bar{\mathcal R}, D, h_0, \gamma),
\]
where $\mathcal H'$ is the set of reachable histories of $\mathcal P'$, and for any
$h\in\mathcal H'$ and $a\in\mathcal A'$,
\begin{align}
    \bar{\mathcal R}(h,a)
    &:= \sum_{s\in\mathcal S} B(s,h)\,\mathcal R'(s,a), \\
    \bar{\mathcal T}(h,a,hao)
    &:= \sum_{s\in\mathcal S}\sum_{s'\in\mathcal S}
    B(s,h)\,\mathcal T'(s,a,s')\,\mathcal Z'(s',o,a),
\end{align}
with $B(\cdot,h)$ denoting the exact belief induced by $h$ in $\mathcal P'$. Define $\kappa_i = \kappa$ for closed-loop actions, and $\kappa_i = 0$ for open-loop actions.

The deterministic-transition analysis above extends to stochastic transitions by
the same meta-node construction as in Lemma~A.1 of \citet{shah2022polyuct}.
Since $\mathcal P'$ is a finite-horizon POMDP and $B(\cdot,h)$ is the exact belief,
the derived history-MDP equivalence of Lemmas~1 and~2 in \citet{silver2010pomcp}
applies.

Let $\tilde V_D^\kappa(h_0)$ denote the depth-$D$ value at the root of
$\tilde{\mathcal M}'$ under the same $\kappa$-adaptive selection rule induced by the VOI-MCTS recursion. Since Lemma~\ref{lemma: recursive convergence} is centered at the
raw mean of the action selected by this $\kappa$-adaptive rule, the root target satisfies
\[
\mu_*^{(0)}(s_0)=\tilde V_D^\kappa(h_0).
\]
By Lemma~1 of \citet{silver2010pomcp}, the value of any history-dependent policy in
$\mathcal P'$ is equal to its value in the derived history-MDP $\tilde{\mathcal M}'$.
Therefore,
\[
\tilde V_D^\kappa(h_0)=\hat V_D'(b_0).
\]
Furthermore, by Proposition~\ref{prop: pprime optimal value},
\[
\hat V_D'(b_0)=\hat V_D^*(b_0).
\]

Moreover, by Lemma~2 of \citet{silver2010pomcp}, the rollout distribution in
$\mathcal P'$ is equal to the rollout distribution in $\tilde{\mathcal M}'$ for any
history-dependent simulation policy. Since VOIMCP in $\mathcal P'$ and VOI-MCTS in
$\tilde{\mathcal M}'$ use the same history-dependent selection rule, their simulation
traces are equal in distribution. Consequently, the VOIMCP root estimator
\[
\bar V_n(b_0)=\frac{1}{n}\sum_{a\in\mathcal A'} N(h_0a)\bar Q(h_0a)
\]
has the same distribution as the root empirical average
\[
\frac{1}{n}\tilde v^{(0)}(s_0)_n
\]
analyzed above. Hence,
\[
\left|\mathbb E[\bar V_n(b_0)]-\hat V_D^*(b_0)\right|
\le O(n^{\eta-1}),
\]
which proves Theorem~\ref{thm:VOIMCP convergence}.

Theorem~\ref{thm:VOIMCP convergence} establishes that VOIMCP converges to the
adaptive VOI value $\hat V_D^*$, while Theorem~\ref{thm:suboptimality} bounds how
far $\hat V_D^*$ can deviate from the true optimal value $V_D^*$. Combining the
two yields the following bound.

\begin{corollary}
\[
\left|
\mathbb E[\bar V_n(b_0)]-V_D^*(b_0)
\right|
\le
O\!\left(n^{\eta-1}\right)
+
\kappa \frac{R_{\max}}{1-\gamma}
\left(\frac{1-\gamma^D}{1-\gamma}\right).
\]
\end{corollary}

\begin{proof}
By the triangle inequality,
\begin{align*}
    \left|
    \mathbb E[\bar V_n(b_0)]-V_D^*(b_0)
    \right|
    \le &
    \left|
    \mathbb E[\bar V_n(b_0)]-\hat V_D^*(b_0)
    \right|
     \nonumber \\ & +
    \left|
    \hat V_D^*(b_0)-V_D^*(b_0)
    \right|.
\end{align*}
The first term is $O(n^{\eta-1})$ by Theorem~\ref{thm:VOIMCP convergence}, and
the second term is bounded by Theorem~\ref{thm:suboptimality}.
\end{proof}

\section{Proof of Theorem~\ref{thm: voimcp global convergence}}

The proof of Theorem~\ref{thm: voimcp global convergence} follows the same structure as Theorem \ref{thm:VOIMCP convergence}. Thus, our analysis relies on a sequence of theoretical results analogous to Lemmas $1-7$ and Theorems $1$ and $3$ in \citet{shah2022polyuct}. From Theorem~\ref{thm:VOIMCP convergence}, the main modifications are the optimality criterion and a scheduled decaying $\kappa_i$. We provide the statements with updated variables below, but detail the full proofs only where the deviation is non-trivial.

We begin again by considering the same non-stationary multi-armed bandit defined in \citet{shah2022polyuct}. Let there be $K \geq 1$ arms or actions. Let $\bar{X}_{i,n} = \frac{1}{n}\sum_{t=1}^{n}X_{i,t}$ denote the empirical average of choosing arm $i$ for $n$ times, and let $\mu_{i,n}$ be its expectation. For each arm $i \in [K]$, the reward $X_{i,t}$ is bounded in $[-R, R]$ for some $R > 0$, with the reward sequence being a nonstationary process satisfying \eqref{eq: MAB convergence} and \eqref{eq: MAB concentration}.

Distinct from Theorem~\ref{thm:VOIMCP convergence}, we now consider the optimal value with respect to the converged expectation of each arm,
\begin{align}
    \mu_* &= \max_{i \in [K]} \mu_i
\end{align}
with the suboptimality gap
\begin{align}
    \Delta_i = \mu_* - \mu_i.
\end{align}

As stated in the main text, any annealing schedule satisfying $0 \leq \kappa_{i}(t, s_i) \leq \frac{1}{c_\kappa R}B_{t, s_i}$ suffices. Define the scheduled VOI-UCB dependent on a time-varying $\kappa$ as
\begin{align}
    \label{eq: scheduled kappa UCB}
    U'(s_i, t) \equiv \bar X_{i,s_i} - \kappa_i(t,s_i) \lvert \bar X_{i,s_i} \rvert+ B_{t,s_i},
\end{align}
Then we define the annealed VOI-UCB algorithm as:

\begin{align}
    I_t \in \arg\max_{i \in [K]} U'_i(s_i,t),
    \label{eq:annealed_action_selection_policy}
\end{align}
where a tie between arms is broken arbitrarily. All other terms are as defined in \citet{shah2022polyuct}.

With the above, we have that the probability that a suboptimal arm or action has a large upper confidence is polynomially small.

\begin{lemma}
\label{lemma: VOI-UCB bounded probability kappa scheduled}
Let $i \in [K]$ be a suboptimal arm. Define
\begin{align}
    A'_i(t) \equiv \min_{u \in \mathbb{N}}\Big\{B_{t,u} \le \frac{\Delta_i}{2} \Big\} 
    = \left\lceil \left(\frac{2}{\Delta_i} \cdot \beta^{1/\xi} \cdot t^{\alpha/\xi}\right)^{\frac{1}{1-\eta}} \right\rceil
\end{align}
For any $t \ge 1$, $t \ge s_i \ge A'_i(t)$,
\begin{align}
    \mathbb{P}\big(U'_i(s_i,t) > \mu_*\big) \le t^{-\alpha}.
\end{align}
\end{lemma}
\begin{proof}
Fix $t \ge 1$ and $s_i \ge A'_i(t)$. Since $\kappa_i(t,s_i) \ge 0$, the penalty 
term $\kappa_i(t,s_i)|\bar{X}_{i,s_i}|$ is non-negative, so
\begin{align*}
    U'_i(s_i,t) = \bar{X}_{i,s_i} - \kappa_i(t,s_i)|\bar{X}_{i,s_i}| + B_{t,s_i} 
    \le \bar{X}_{i,s_i} + B_{t,s_i}.
\end{align*}
Therefore,
\begin{align*}
    \mathbb{P}\big(U'_i(s_i,t) > \mu_*\big) 
    &\le \mathbb{P}\big(\bar{X}_{i,s_i} + B_{t,s_i} > \mu_*\big) \\
    &= \mathbb{P}\big(\bar{X}_{i,s_i} - \mu_i > \Delta_i - B_{t,s_i}\big).
\end{align*}
Since $s_i \ge A'_i(t)$, by definition $B_{t,s_i} \le \Delta_i/2$, so 
$\Delta_i - B_{t,s_i} \ge \Delta_i/2 \ge B_{t,s_i}$. Hence,
\begin{align*}
    \mathbb{P}\big(U'_i(s_i,t) > \mu_*\big) 
    \le \mathbb{P}\big(\bar{X}_{i,s_i} - \mu_i > B_{t,s_i}\big) 
    \le t^{-\alpha},
\end{align*}
where the last step uses the concentration assumption~\eqref{eq: MAB concentration}.
\end{proof}

As in Theorem \ref{theorem: MAB VOI-UCB Convergence}, let $i^* \in \arg\max_{i \in [K]} \mu_i$ be the optimal arm. Then we have the following lemma.

\begin{lemma}
\label{lemma: optimal arm lower bound}
For any $t \geq 1$ and $1 \leq s \leq t$,
\begin{align*}
    \mathbb{P}\big(U'_*(s_*,t) \leq \mu_*\big) \leq \Big(\frac{c_\kappa}{c_\kappa-1}\Big)^\xi t^{-\alpha}.
\end{align*}
\end{lemma}
\begin{proof}
Since $\kappa_*(t,s_*) \leq \frac{1}{c_\kappa R} B_{t,s_*}$ and 
$|\bar{X}_{*,s_*}| \leq R$, the penalty satisfies
\begin{align*}
    \kappa_*(t,s_*)|\bar{X}_{*,s_*}| \leq \frac{1}{c_\kappa R} B_{t,s_*} \cdot R 
    = \frac{B_{t,s_*}}{c_\kappa}.
\end{align*}
Therefore,

\begin{align*}
    B_{t,s_*} - \kappa_*(t,s_*) \lvert \bar X_{*,s_*} \rvert \geq B_{t,s_*} \Big(\frac{c_\kappa-1}{c_\kappa}\Big)
\end{align*}
Hence,

\begin{align*}
    \mathbb{P}\big(U'_*(s_*,t) &\leq \mu_*\big) \\
    &=\mathbb{P}\big(\bar{X}_{*,s_*} + B_{t,s_*} - \kappa_*(t,s_*) \lvert \bar X_{*,s_*} \rvert \leq \mu_*\big) \\
    & \leq \mathbb{P}\big(\bar{X}_{*,s_*} + B_{t,s_*}\Big(\frac{c_\kappa-1}{c_\kappa}\Big) \leq \mu_* \big) \\
    & = \mathbb{P}\big(s(\bar{X}_{*,s_*} - \mu_*) \leq - s\cdot B_{t,s_*} \Big(\frac{c_\kappa-1}{c_\kappa}\Big)\big)
\end{align*}
To apply the concentration assumption~\eqref{eq: MAB concentration}, we match $s^\eta z = s\cdot B_{t,s_*} ((c_\kappa-1)/c_\kappa)$, giving
\begin{align*}
    z &= s^{1-\eta} B_{t,s_*}\Big(\frac{c_\kappa-1}{c_\kappa}\Big) = s^{1-\eta} \cdot \frac{\beta^{1/\xi} t^{\alpha/\xi}}{s^{1-\eta}} \Big(\frac{c_\kappa-1}{c_\kappa}\Big) \\
    &= \beta^{1/\xi} t^{\alpha/\xi}\Big(\frac{c_\kappa-1}{c_\kappa}\Big).
\end{align*}
To ensure $z \geq 1$, without loss of generality, let $\beta$ be large enough that $\beta^{1/\xi} \geq (c_\kappa/(c_\kappa-1))$ (cf.\ the discussion below Eq.~(20) in \citet{shah2022polyuct}). Applying~\eqref{eq: MAB concentration}:
\begin{align*}
    \mathbb{P}\big(U'_*(s_*,t) \leq \mu_*\big) 
    \leq \beta z^{-\xi} 
    &= \beta \cdot \left(\beta^{1/\xi} t^{\alpha/\xi}\Big(\frac{c_\kappa-1}{c_\kappa}\Big)\right)^{-\xi} \\
    &= \Big(\frac{c_\kappa}{c_\kappa-1}\Big)^\xi t^{-\alpha}.
\end{align*}
\end{proof}

Lemma~\ref{lemma: VOI-UCB bounded probability kappa scheduled} allows us to upper bound the expected number of suboptimal arm pulls as follows.

\begin{lemma}
\label{lemma: scheduled alpha VOI-UCB bounded expected pulls}
Let $i \in [K]$, $i \neq i^*$ be a suboptimal arm. Then, for all $n \ge 1$, the expected number of pulls of arm $i$ satisfies
\begin{align*}
\mathbb{E}[T_i(n)]
    &\le \left(\frac{2}{\Delta_i} \cdot \beta^{1/\xi}\right)^{\frac{1}{1-\eta}} n^\frac{\alpha}{\xi(1-\eta)}  + \frac{1 + \left(\frac{c_\kappa}{c_\kappa - 1}\right)^\xi}{\alpha-2} + 1.
\end{align*}
\end{lemma}

\begin{proof}
    The proof follows the structure of Lemma~\ref{lemma: VOI-UCB bounded expected pulls}. 
    Event $E_1$ (suboptimal arm's VOI-UCB exceeds $\mu_*$) is bounded 
    using Lemma~\ref{lemma: VOI-UCB bounded probability kappa scheduled}. 
    Event $E_2$ (optimal arm's VOI-UCB falls below $\mu_*$) is bounded 
    using Lemma~\ref{lemma: optimal arm lower bound}. 
    The remainder follows identically, with the constant $(1 + (c_\kappa/(c_\kappa - 1))^\xi) / (\alpha-2)$ 
    absorbing the factor from the scaled bonus.
\end{proof}

Now we can state the convergence and concentration results for the annealed VOI nonstationary multi-armed bandit. Let $\bar X_n \equiv \frac{1}{n} \sum_{i=1}^K T_i(n)\bar X_{i,T_i(n)}$ denote the empirical average under the annealed VOI-UCB algorithm \eqref{eq:annealed_action_selection_policy}. Then, $\bar X_n$ satisfies the following convergence and concentration properties.
\begin{theorem}
\label{theorem: MAB VOI-UCB scheduled kappa Convergence}
Consider a nonstationary MAB satisfying~\eqref{eq: MAB convergence} and~\eqref{eq: MAB concentration}. 
Suppose that Algorithm~\eqref{eq:annealed_action_selection_policy} is applied with parameter 
$\alpha$ such that $\xi \eta (1-\eta) \le \alpha < \xi(1-\eta)$ and $\alpha > 2$. 
Then, the following holds:

\noindent \textbf{A. Convergence}:
{\small
\begin{align*}
&\big|\mathbb{E}[\bar X_n] - \mu^*\big|\\
&
\le \lvert \delta^*_{n} \rvert \\
&+ \frac{2R(K-1)}{n} 
\!\left(
\!\!\Big(\frac{2}{\Delta_{\min}} \beta^{\frac{1}{\xi}}\Big)^{\frac{1}{1-\eta}} \!\cdot\! n^{\frac{\alpha}{\xi(1-\eta)}} \!+ \!\frac{1 + \left(\frac{c_\kappa}{c_\kappa - 1}\right)^\xi}{\alpha-2} \!+\! 1
\!\right).
\end{align*}
}
\noindent \textbf{B. Concentration}: there exists constants, $\beta' > 1$ and $\xi' > 0$ and $1/2 \leq \eta' < 1$ such that for every $n \geq 1$ and every $z \geq 1$,
\begin{align*}
    \mathbb{P}(n\bar{X}_n - n\mu_* \geq n^{\eta'}z) \leq \frac{\beta'}{z^{\xi'}},\\
    \mathbb{P}(n\bar{X}_n - n\mu_* \leq -n^{\eta'}z) \leq \frac{\beta'}{z^{\xi'}},
\end{align*}
where $\eta' = \frac{\alpha}{\xi(1-\eta)}, \xi' = \alpha - 1, \beta'$ is a large enough constant that depends on $R, K, \Delta_{min}, \beta, \xi, \alpha, \eta$. 
\end{theorem}
\begin{proof}
    The proof follows directly from the proof of Theorem~\ref{theorem: MAB VOI-UCB Convergence}, by substituting the global sample threshold:

    \begin{equation}
        \begin{aligned}
            A'(t) &= \max_{i \in [K]}A'_i(t) 
            = \left\lceil \left( \frac{2}{\Delta_{\min}} \cdot \beta^{1/\xi} \right)^{1/(1-\eta)} t^{\alpha/(\xi(1-\eta))} \right\rceil.
        \end{aligned}
    \end{equation}
\end{proof}

\subsubsection{Analyzing Leaf Level $d = D$}

By applying Lemma 4 from \citet{shah2022polyuct}, we obtain that for $\eta^{(D)} \in \left[\tfrac{1}{2}, 1\right)$, and sufficiently large $\xi^{(D)}$, 
there exists a constant $\beta^{(D)}$ such that the collected rewards at 
$s_{i,D-1}$ (i.e., the sum of the i.i.d. stochastic reward and the value 
evaluation) satisfy the convergence property (cf.~\eqref{eq: MAB convergence}) and the concentration 
property (cf.~\eqref{eq: MAB concentration}). The optimal mean target $\mu_*^{(D\mhyphen1)}(s_{i,D\mhyphen1})$ and $\Delta_{min}^{(D \mhyphen 1)}(s_{i,D\mhyphen1})$ are as defined in \citet{shah2022polyuct}. Then, by an application of Theorem~\ref{theorem: MAB VOI-UCB scheduled kappa Convergence}, we have the following Lemma. 

\begin{lemma}
    Consider a node corresponding to history $s_{i,D-1}$ at level $D-1$ within the VOIMCP for $i \in [n_{D-1}]$. Let $\tilde{v}^{(D-1)}(s_{i,D-1})_n$ be the total discounted reward collected at $s_{i,D-1}$ during $n \ge 1$ visits. Then, for the choice of appropriately large $\beta^{(D)} > 0$, for a given $\xi^{(D)} > 0$, $\eta^{(D)} \in [\frac{1}{2}, 1)$ and $\alpha^{(D)} > 2$, we have

\noindent \textbf{A. Convergence:}
{\small
\begin{align*}
    &\left|E\left[\frac{1}{n} \tilde{v}^{(D-1)}(s_{i,D-1})_n\right] - \mu^{(D-1)}_{*}(s_{i,D-1})\right| \\
    &\le \psi^{(D-1)}_n\!\! \left(\left( \frac{2(\beta^{(D)})^{\frac{1}{\xi^{(D)}}}}{\Delta_{\min}^{(D-1)}(s_{i,D-1})}\right)^{\frac{1}{1-\eta^{(D)}}} \!\!\cdot n^{\frac{\alpha^{(D)}}{\xi^{(D)}(1-\eta^{(D)})}} + C^{(D)} \right),
\end{align*}
}
where $\psi^{(D-1)}_n = \frac{2 \tilde{R}_{\max}^{(D-1)} (K-1)}{n}$ and $C^{(D)} = \frac{1 + \left(\frac{c_\kappa}{c_\kappa - 1}\right)^{\xi^{(D)}}}{\alpha^{(D)}-2} + 1$.

\noindent \textbf{B. Concentration:} there exist constants, $\beta' > 1$ and $\xi' > 1$ and $1/2 \le \eta' < 1$ such that for $n \ge 1$, $z \ge 1$,
\begin{align*}
    \mathbb{P}(\tilde{v}^{(D-1)}(s_{i,D-1})_n - n\mu^{(D-1)}_{*}(s_{i,D-1}) &\ge n^{\eta'} z) \le \frac{\beta'}{z^{\xi'}}, \\
    \mathbb{P}(\tilde{v}^{(D-1)}(s_{i,D-1})_n - n\mu^{(D-1)}_{*}(s_{i,D-1}) &\le -n^{\eta'} z) \le \frac{\beta'}{z^{\xi'}},
\end{align*}
where $\eta' = \frac{\alpha^{(D)}}{\xi^{(D)}(1-\eta^{(D)})}$, $\xi' = \alpha^{(D)} - 1$, and $\beta'$ is a large enough constant that is function of parameters $\alpha^{(D)}, \beta^{(D)}, \xi^{(D)}, \eta^{(D)}, \tilde{R}_{\max}^{(D-1)}, K, \Delta_{\min}^{(D-1)}(s_{i,D-1})$.
\end{lemma}

\subsubsection{Recursion: Going from level $d$ to $d-1$}

The proof for the convergence and concentration properties of the value estimates across all levels $d$ of the VOIMCP tree proceeds via mathematical induction, following the structure in \citet{shah2022polyuct}. The mechanism of the induction is identical to the original paper, but the final, global convergence rate differs due to how we defined the sample threshold $A'_{i}(t)$.

Assume the estimated value $\tilde{v}^{(d)}(\cdot)$ at level $d$ satisfies properties \eqref{eq: MAB convergence} and \eqref{eq: MAB concentration} with parameters $\alpha^{(d)}$, $\xi^{(d)}$, $\eta^{(d)}$, and an appropriately large constant $\beta^{(d)}$. For any node $s_{i,d-1}$ at level $d-1$, the effective reward received is the combined value $R(s_{i,d-1}, a) + \gamma \tilde{v}^{(d)}(s'_{d})$. By the inductive hypothesis and an application of Lemma $4$ of \citet{shah2022polyuct}, this combined reward sequence continues to satisfy assumptions \eqref{eq: MAB convergence} and \eqref{eq: MAB concentration} with the same $\alpha, \xi, \eta$ dependencies and a re-adjusted constant $\beta^{(d-1)}$.

Under the same bounded reward, value definitions, and objects $\mu_*^{(d\mhyphen1)}(s_{i,d\mhyphen1})$ and $\Delta_{min}^{(d \mhyphen 1)}(s_{i,d\mhyphen1})$ defined in \citet{shah2022polyuct}, we apply Theorem~\ref{theorem: MAB VOI-UCB scheduled kappa Convergence} and obtain the following result regarding the desired convergence and concentration properties of the combined reward sequence at level $d-1$, which completes the induction and proves the results for all $d \geq 0$.

\begin{lemma}
    \label{lemma:anneal_d}
    Consider a node corresponding to state $s_{i,d-1}$ at level $d-1$ within the VOIMCP for $i \in [n_{d-1}]$. Let $\tilde{v}^{(d-1)}(s_{i,d-1})_n$ be the total discounted reward collected at $s_{i,d-1}$ during $n \ge 1$ visits. Then, for the choice of appropriately large $\beta^{(d)} > 0$, for a given $\xi^{(d)} > 0$, $\eta^{(d)} \in [\frac{1}{2}, 1)$ and $\alpha^{(d)} > 2$, we have

\noindent \textbf{A. Convergence:}
{\small
\begin{align*}
    &\left|E\left[\frac{1}{n} \tilde{v}^{(d-1)}(s_{i,d-1})_n\right] - \mu^{(d-1)}_{*}(s_{i,d-1})\right| \\
    &\le \psi^{(d-1)}_{n}\!\! \left( \left( \frac{2(\beta^{(d)})^{\frac{1}{\xi^{(d)}}}}{\Delta_{\min}^{(d-1)}(s_{i,d-1})}\right)^{\frac{1}{1-\eta^{(d)}}} \!\!\cdot n^{\frac{\alpha^{(d)}}{\xi^{(d)}(1-\eta^{(d)})}} + C^{(d)} \right),
\end{align*}
}
where $\psi^{(d-1)}_{n} = \frac{2 \tilde{R}_{\max}^{(d-1)} (K-1)}{n}$ and $C^{(d)} = \frac{1 + \left(\frac{c_\kappa}{c_\kappa - 1}\right)^{\xi^{(d)}}}{\alpha^{(d)}-2} + 1$.

\noindent \textbf{B. Concentration:} there exist constants, $\beta' > 1$ and $\xi' > 1$ and $1/2 \le \eta' < 1$ such that for $n \ge 1$, $z \ge 1$,
\begin{align*}
    \mathbb{P}(\tilde{v}^{(d-1)}(s_{i,d-1})_n - n\mu^{(d-1)}_{*}(s_{i,d-1}) &\ge n^{\eta'} z) \le \frac{\beta'}{z^{\xi'}}, \\
    \mathbb{P}(\tilde{v}^{(d-1)}(s_{i,d-1})_n - n\mu^{(d-1)}_{*}(s_{i,d-1}) &\le -n^{\eta'} z) \le \frac{\beta'}{z^{\xi'}},
\end{align*}
where $\eta' = \frac{\alpha^{(d)}}{\xi^{(d)}(1-\eta^{(d)})}$, $\xi' = \alpha^{(d)} - 1$, and $\beta'$ is a large enough constant that is function of parameters $\alpha^{(d)}, \beta^{(d)}, \xi^{(d)}, \eta^{(d)}, \tilde{R}_{\max}^{(d-1)}, K, \Delta_{\min}^{(d-1)}(s_{i,d-1})$.
\end{lemma}

\subsection{Completing Proof of Theorem 3.}
Using Lemma \ref{lemma:anneal_d}, we conclude that the recursive relationship going from $d$ to $d-1$ holds for all $d\geq1$ with level $0$ being the root. Thus, for the root state $s^{(0)}$, under $n$ total simulations of annealed VOI-MCTS, the empirical average of the rewards $\frac{1}{n}\tilde v^{(0)}(s_0)_n$ is such that (using the fact that $\alpha^{(0)} \equiv \xi^{(0)} (1 - \eta^{(0)}) / \eta^{(0)}$)

\begin{align}
    \left|\mathbb{E}\left[\frac{1}{n} \tilde{v}^{(0)}(s_{0})_n\right] - \mu^{(0)}_{*}(s_{0})\right| &\le O\left( n^{\frac{\alpha^{(0)}}{\xi^{(0)} (1-\eta^{(0)})} - 1} \right) \nonumber \\
    &= O(n^{\eta-1})
\end{align}
where $\mu^{(0)}_{*}(s_{0})$ is the value for $s_0$ after $D$ iterations of value iteration. We seek to transfer this root-level VOI-MCTS result to the VOIMCP estimator
$\bar{V}_n(b_0)=\frac{1}{n}\sum_{a\in\mathcal{A}'}N(h_0a)\bar Q(h_0a)$. 

The deterministic-transition analysis above extends to stochastic transitions by
the same meta-node construction as in Lemma~A.1 of \citet{shah2022polyuct}.
Since $\mathcal P'$ is a finite-horizon POMDP, Lemmas~1 and~2 of
\citet{silver2010pomcp} apply directly. Recall the derived history-MDP $\tilde{\mathcal M}'$ for $\mathcal{P}'$ defined in the proof of Theorem \ref{thm:VOIMCP convergence}. 

Let $\tilde V_D^*(h_0)$ denote the ordinary depth-$D$ optimal value at the root of
$\tilde{\mathcal M}'$. Since Lemma~\ref{lemma:anneal_d} is centered at the raw
optimal mean, the root target satisfies
\[
\mu_*^{(0)}(s_0)=\tilde V_D^*(h_0).
\]

By Lemma~1 of \citet{silver2010pomcp}, the value of any history-dependent policy
in $\mathcal P'$ equals its value in the derived history-MDP $\tilde{\mathcal M}'$.
Moreover, because the augmented action space $\mathcal A'=\mathcal A^{OL}\cup
\mathcal A^{CL}$ contains a closed-loop counterpart of every original action and
closed-loop backups weakly dominate open-loop backups, the ordinary optimal value
in $\mathcal P'$ is exactly the original POMDP optimal value. Hence
\[
\tilde V_D^*(h_0)=V_D^*(b_0).
\]

By Lemma 2 of \citet{silver2010pomcp}, the simulation traces generated by annealed VOIMCP in 
$\mathcal{P}'$ are equal in distribution to those generated by annealed VOI-MCTS in 
$\tilde{\mathcal{M}}'$. Therefore, the root estimator produced by annealed VOIMCP after $n$ simulations 
has the same distribution as the root estimator analyzed above. Thus, we conclude 
that the output $\bar{V}_n(b_0)$ of annealed VOIMCP with $n$ simulations satisfies
\begin{align}
    \left| \mathbb{E}[\bar{V}_n(b_0)] - V^*_D(b_0) \right| \leq \mathcal{O}(n^{\eta - 1}),
\end{align}
thus proving the theorem.

\section{Supplement on Empirical Evaluation}
\subsection{Implementation Details of VOIMCP}

While the main text presents the theoretical recursive formulation of VOIMCP for convergence analysis, our practical implementation utilizes the standard PO-UCT expansion and rollout strategy for computational efficiency. Algorithm \ref{alg:voi-mcp-impl} details this implementation. The primary distinction is the termination of the \textsc{simulate} procedure and returning a value estimate (e.g., with a heuristic rollout policy) in Line $17$ upon expanding a new history node $h$.

\subsection{Experimental Protocol}

Each experimental trial refers to one full environment episode rollout under an online replanning framework, starting from an initial state/belief, and terminating under the task's terminal state or when the horizon is reached. We plan by interleaving policy planning and execution, using a standard SIR particle filter to update the belief before each replanning step. Each trial uses an independent seed, controlling any randomness in the trial, e.g., sampling from the generative model, particle filtering.

\begin{algorithm}[t]
\caption{VOIMCP (Implementation)}
\label{alg:voi-mcp-impl}
\small
\vspace{0pt}
\begin{algorithmic}[1]
\Procedure{Search}{$b$}
  \ForAll{$i = 1,\dots,n$}
    \State $s\sim b$
    \State $h_0 = \{b\}$
    \State \Call{Simulate}{$s,h_0,0$}
  \EndFor
  \State \Return $\displaystyle
    \arg\max_{a\in\mathcal{A}'}\bar Q(ha)$   
  
\EndProcedure
\vspace{2pt}
\Procedure{Simulate}{$s,h,depth$}
  \If{$depth=D$} \State \Return $0$ \EndIf
  \If{$h\notin T$}
    \ForAll{$a\in\mathcal{A'}$}
      \State $T(ha)\leftarrow(0,0)$
    \EndFor
    \State \Return \Call{ValueEstimate}{$s,h,depth$}
  \EndIf
  \State $a^*\!\in\underset{a\in\mathcal{A'}}{\!\arg\max}\; UCB_{VOI}(ha)$
  \If{$a^* \in \mathcal{A}_{CL}$}
  \State $(s',o,r)\sim \mathcal{G}(s,a^*)$
  \Else
  \State $(s',r)\sim \mathcal{G}(s,a^*)$
  \State $o\leftarrow o_{null}$ 
  \EndIf
  \State $N(h)\leftarrow N(h)+1$; $N(ha^*)\leftarrow N(ha^*)+1$
  \State $R\leftarrow r+\gamma\Call{Simulate}{s',ha^*o,depth{+}1}$
  \State $\bar Q(ha^*)\leftarrow \bar Q(ha^*)+\dfrac{R - \bar Q(ha^*)}{N(ha^*)}$
  \State \Return $R$
\EndProcedure
\end{algorithmic}
\end{algorithm}

\subsection{Hyperparameters}

Table~\ref{tab:hyperparams} summarizes the best-performing hyperparameters identified through grid search. We observed that VOIMCP is relatively robust to the exploration constant $c$. Across all three benchmark domains, we found that VOIMCP performance is relatively robust to the choice of $\kappa$, with values in the range $[0.01,0.2]$ consistently outperforming the baselines.

\begin{table}[ht!]
    \centering
    \setlength{\tabcolsep}{4pt}
    \renewcommand{\arraystretch}{1.3}
    \begin{tabular}{l l c c c}
        \hline
        Domain & Algorithm & $c$ & $\kappa$ & $D$ \\
        \hline
        Tracking              & VOIMCP    & 100 & 0.03 & 20 \\
                      & PO-UCT   & 100 & - & 20 \\
                      & I-UCB POMCP & 100 & - & 20 \\
                      & Open-Loop & 100 & - & 20 \\
        \hline
        FVRockSample  & VOIMCP    & 10 & 0.02 & 20 \\
                      & PO-UCT    & 10 & - & 20 \\
                      & I-UCB POMCP & 10 & - & 20 \\
                      & Open-Loop & 10 & - & 20 \\
        \hline
        Laser Tag               & VOIMCP    & 100 & 0.01 & 20 \\
                       & PO-UCT   & 100 & - & 20 \\
                       & I-UCB POMCP & 100 & - & 20 \\
                      & Open-Loop & 100 & - & 20 \\
        \hline
    \end{tabular}
    \caption{Summary of hyperparameters used in experiments.}
    \label{tab:hyperparams}
\end{table}

\subsection{Annealing Results}

\begin{figure}[t!]
  \centering
  \begin{subfigure}[b]{\linewidth}
    \centering
    \includegraphics[width=0.73\linewidth]{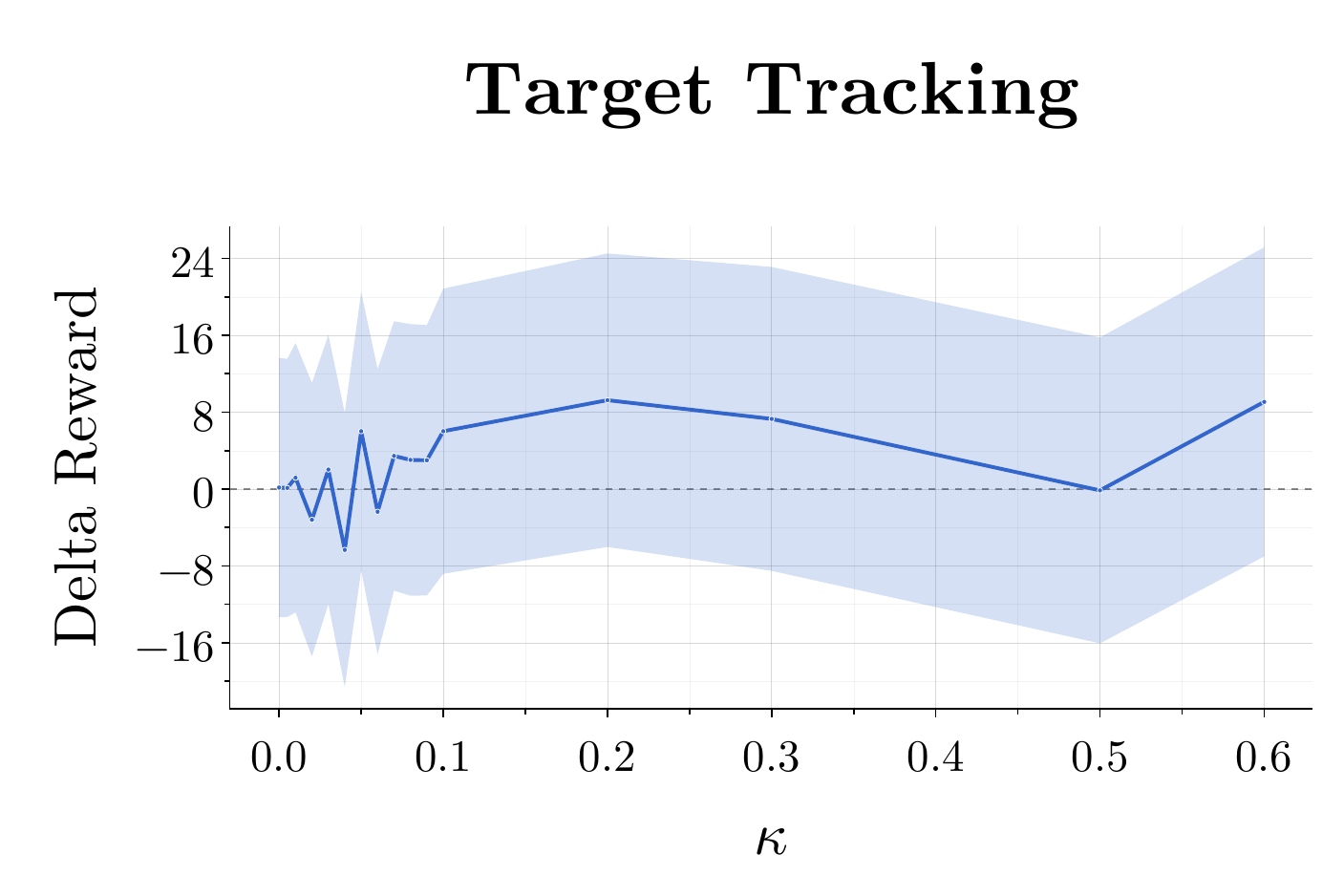}
    \end{subfigure}
  \begin{subfigure}[b]{\linewidth}
    \centering
    \includegraphics[width=0.73\linewidth]{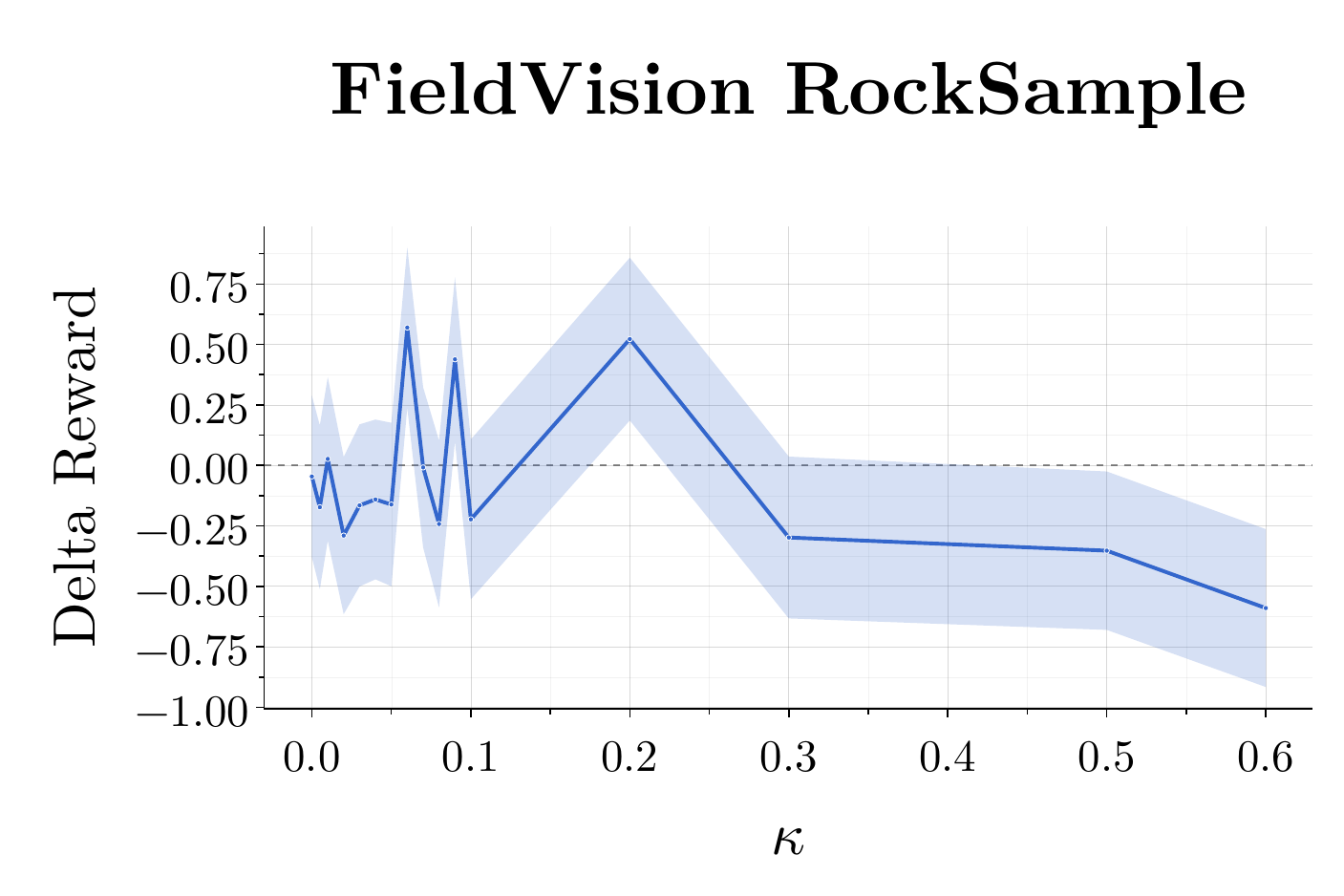}
    \end{subfigure}
  \begin{subfigure}[b]{\linewidth}
    \centering
    \includegraphics[width=0.73\linewidth]{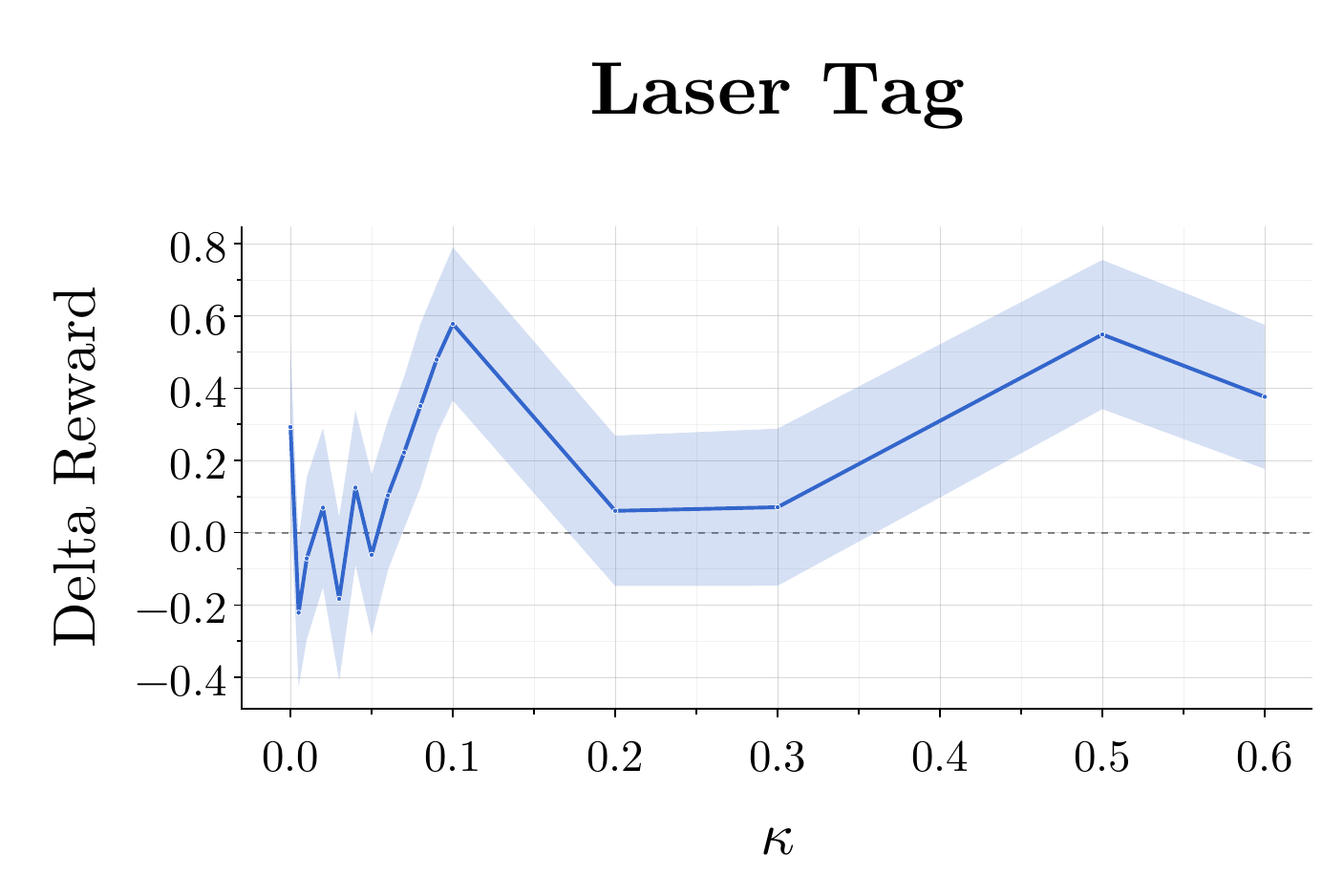}
    \end{subfigure}
  \caption{Results for VOIMCP difference between the annealed and standard algorithms, presenting delta discounted cumulative reward over $1000$ trials. The lighter-colored ribbons around the Monte Carlo mean display the $95\%$ confidence interval.}
  \label{fig:anneal results}
\end{figure}

We additionally evaluate the performance of VOIMCP using a simplified count-based schedule where $\kappa$ decays as $\kappa_N=\kappa_{max}\frac{x}{1+x}$ where $x=\frac{\kappa}{\kappa_{max}}B_N$. As visit counts grow, $\kappa$ decays proportionally to $B_N$ and is smoothly bounded in [0, $\kappa_{max}$). Figure~\ref{fig:anneal results} compares the discounted cumulative reward difference between the fixed-$\kappa$ implementation and the annealed variant across all three benchmarks (annealed minus standard).

These results show that this annealing schedule is not only theoretically rigorous but also practically beneficial. Over a range of $\kappa_{max}$, when compared to a fixed $\kappa_{max}$ baseline, the annealed VOIMCP variant achieved peak performance improvements of approximately 2.52\%, 4.06\%, and 3.44\% for the Target Tracking, FieldVision RockSample, and Laser Tag domains, respectively. While these gains are modest, these results serve as a valuable proof of concept: annealing can provide a path to asymptotic optimality while reducing reliance on careful tuning of a fixed $\kappa$. A more exhaustive analysis of complex annealing schedules and their empirical trade-offs remains a subject for future work.

\end{document}